\documentclass[11pt]{article}

\usepackage[T1]{fontenc}
\usepackage[utf8]{inputenc}
\DeclareUnicodeCharacter{2011}{-}

\usepackage{amsmath,amssymb}
\usepackage{times}
\usepackage{bm}
\usepackage{natbib}
\usepackage{algorithm}
\usepackage{mathrsfs}
\usepackage{graphicx}
\usepackage[flushleft]{threeparttable}
\usepackage[shortlabels]{enumitem}
\usepackage{titletoc}
\usepackage{tocloft}
\usepackage{nexais}
\makeatother
\usepackage{microtype}
\usepackage{cmap}
\usepackage{comment}
\usepackage{booktabs,float}
\usepackage{rotating}
\usepackage{tikz}
\usepackage{array}
\usepackage{tabularx}
\usepackage{mathtools}
\usepackage[most]{tcolorbox}
\usepackage{marginnote}

\usepackage[colorlinks,
linkcolor=metablue,
anchorcolor=metablue,
citecolor=metablue,
urlcolor=metablue
]{hyperref}

\usepackage{txfonts}
\usepackage{geometry}
\providecommand{\e}{\varepsilon}

\providecommand{\I}{\frac{1}{\theta}}

\providecommand{\DD}{\mathcal{D}}

\providecommand{\al}{\alpha}

\providecommand{\paren}[1]{\left(#1\right)}
\renewcommand{\norm}[1]{\left\|#1\right\|}
\renewcommand{\abs}[1]{\left|#1\right|}
\providecommand{\bracks}[1]{\left[#1\right]}
\renewcommand{\paragraph}[1]{\noindent\textit{#1}}
\providecommand{\Renyi}{R\'enyi }
\providecommand{\subw}{\operatorname{subW}}

\providecommand{\iid}{\textrm{i.i.d.}}

\providecommand{\cmmnt}[1]{}

\providecommand{\scomment}[1]{}

\begin{document}

\title{Tail-Aware Information-Theoretic Bounds for LLM Alignment under Heavy-Tailed Rewards}

\author[1,*]{Huiming Zhang}
\author[2,*]{Binghan Li}
\author[3]{Wan Tian}
\author[4]{Qiang Sun}

\contribution[*]{Co-first authors.}

\affiliation[1]{BUAA and BAIC-FBPC}

\affiliation[2]{BUAA}

\affiliation[3]{PKU}

\affiliation[4]{UofT and MBZUAI}
%
%
%


\abstract{

Classical information-theoretic learning bounds typically rely on KL mutual
information and moment-generating-function (MGF) arguments, which are well
matched to bounded or sub-Gaussian losses but can be ineffective when losses or
rewards are heavy-tailed. We develop a tail-aware information-theoretic
framework for sub-Weibull data, where the tail parameter $\theta$ controls the
tail heaviness: $\theta=2$ corresponds to sub-Gaussian, $\theta=1$ to
sub-exponential, and $0<\theta<1$ to genuinely heavy tails. Our key technical
ingredient is a decorrelation lemma that bounds change-of-measure expectations
using a shifted-log $f_\theta$-divergence, which admits explicit comparisons to
R\'enyi divergence without MGF arguments. On the empirical-process side, we
establish sharp maximal inequalities and a Dudley-type chaining bound for
sub-Weibull processes, with logarithmic and entropy terms raised to the power
$1/\theta$. These tools yield tail-adaptive selection bounds and a multiscale
information-theoretic Dudley inequality based on shifted-log and R\'enyi mutual
information. We apply our theory to large language models (LLMs) in the context of reward hacking within reinforcement learning from human feedback (RLHF). We show that Rényi-regularized alignment provides finite reward guarantees and ensures that best-of-N policies remain well-controlled, thereby mitigating the catastrophic Goodhart effects where standard KL-regularization fails. We illustrate R\'enyi-regularized RLHF by experiments, including controlled heavy-tailed rewards and token-space reward attacks.  The code is available at \url{https://github.com/NeXAIS/SubW-Alignment}.

}
\metadata[Keywords]{
heavy-tailed distributions, sub-Weibull processes, R\'enyi divergence, information-theoretic generalization bounds, LLM, RLHF.
}

\metadata[Code]{\url{https://github.com/NeXAIS/SubW-Alignment}}

\date{\today}

\maketitle

\bibliographystyle{apalike}

\tableofcontents

\section{Introduction}\label{intro}



The generalization gap, namely the difference between population risk and
empirical risk, has traditionally been analyzed by uniform-convergence
arguments based on classical complexity measures such as VC dimension and
Rademacher or Gaussian complexity; see
\cite{koltchinskii2011oracle, Wainwright19, shalev2014understanding}.
While powerful, these tools are inherently worst-case: they control an entire
hypothesis class rather than the smaller, sample-dependent region actually
explored by a modern learning algorithm. As a result, they can be overly
pessimistic for data-adaptive procedures, which often generalize well even when
the ambient class has large worst-case complexity \citep{zhou2021optimistic}.

Information-theoretic generalization bounds replace global complexity by an
\emph{algorithm-dependent} quantity: how much information the algorithm's
output reveals about the sample \citep{xu2017,Russo20,bu2020}. In
light-tailed settings, under sub-Gaussian assumptions, KL-based
mutual information yields clean bounds in expectation and with high
probability, and these guarantees can be further sharpened through multiscale
or chaining arguments \citep{asadi2018chaining,hellstrom2025generalization}.

A major obstacle is that KL-based arguments are tightly coupled to moment
generating functions (MGFs). Many modern pipelines, however, exhibit
heavy-tailed behavior, so the relevant MGFs may fail to exist. This issue
arises in reward-model scores for RLHF \citep{kwa2024catastrophic}, gradient noise in large-scale
optimization \citep{raj2023algorithmic}, and losses induced by heavy-tailed data
\citep{lerasle2019lecture, xu2023non}. In such settings, KL mutual information
can remain small even when rare but extreme events dominate generalization or
stability; see Example~\ref{ex:weibull}.


A convenient tail model that interpolates between sub-Gaussian and heavy-tailed
behavior is the \emph{sub-Weibull} family. 
A random variable $X$ is called sub-Weibull with parameter $\theta>0$, denoted
$X\sim\subw(\theta)$, if its Orlicz norm $\|X\|_{\psi_\theta}$ is finite, where
$\psi_\theta(x)=\exp(x^\theta)-1$. Equivalently, there exists $K>0$ such that
$\mathbb{E}\exp(|X|^\theta/K^\theta)\le 2$ or the tail estimate
\begin{equation}\label{tail-estimate}
\mathbb{P}(|X|\ge t)\le 2\exp\bigl(-t^\theta/K^\theta\bigr),\qquad t\ge 0,
\end{equation}
for some $K>0$. The $\theta$ controls tail thickness: $\theta=2$
corresponds to sub-Gaussian tails, $\theta=1$ to sub-exponential tails, and
$0<\theta<1$ to genuinely heavy-tailed regimes.  Our focus is on deriving
information-theoretic generalization tools that remain meaningful throughout the
sub-Weibull family, including $0<\theta<1$ where MGF-based techniques break down.

This work develops an information-theoretic generalization framework for
modern learning algorithms under heavy-tailed distributions and rewards. Our main contributions are:
\begin{enumerate}
\item \textbf{Tail-adaptive decorrelation via shifted-log divergences.}
We introduce a shifted-log $f_\theta$-divergence
$f_\theta(x)=x\log^{1/\theta}(x+A)$ and prove a decorrelation (change-of-measure)
lemma that bounds expectations under a dependent coupling through this
tail-adaptive divergence. We further show that it admits explicit upper bounds
in terms of R\'enyi divergence, leading to usable bounds in terms of R\'enyi
mutual information (Section~\ref{sec:pre}).

\item \textbf{Maximal and chaining bounds for sub-Weibull processes.}
For $0<\theta<1$, we prove a sharp maximal inequality for a finite family of
sub-Weibull random variables and a Dudley-type entropy integral for
$(\theta,C)$-sub-Weibull processes (Section~\ref{sec:ep}).

\item \textbf{Information-theoretic selection and chaining under sub-Weibull tails.}
Combining the decorrelation lemma with the new sub-Weibull maximal and chaining
bounds, we prove tail-adaptive information bounds for data-dependent selectors.
These inequalities replace KL mutual information, which can fail under heavy
tails, by shifted-log $f_\theta$-mutual information and its R\'enyi refinements,
and extend from finite selections to Dudley-type multiscale bounds
(Section~\ref{sec:gen_bounds}).

\item \textbf{Applications to RLHF under heavy-tailed and adversarial rewards.}
We show that KL-constrained RLHF can suffer catastrophic Goodhart under
heavy-tailed proxy rewards, while R\'enyi-regularized RLHF admits tail-adaptive
reward guarantees at a fixed divergence budget. We also analyze best-of-$n$
policies and bound their reward growth under sub-Weibull reward tails
(Section~\ref{sec:rlhf}). Numerically, we study both controlled heavy-tail
perturbations and token-space reward attacks; in both settings, higher-order
R\'enyi regularization dampens tail inflation and improves proxy--gold reward
alignment relative to KL (Sections~\ref{sec:rlhf-experiments}
and~\ref{sec:tompa-experiments}; Appendix~\ref{app:add_experiments}).

\end{enumerate}

\subsection*{Related Work}

\paragraph{Maximal inequalities for heavy tails.}
Maximal inequalities and chaining are classical tools in empirical process
theory \citep{Talagrand2005,vanhandle2016probability}.
While maximal inequalities for sub-Weibull processes are well-understood for
$\theta\ge 1$ \citep{kuchibhotla2022moving}, the heavy-tailed regime $0<\theta<1$ is more
delicate because MGFs can be infinite and many standard
arguments fail.  Recent works provide maximal inequalities with explicit
prefactors \citep{mendes2023concentration} and sharp Orlicz norm bounds
\citep{zhang2022sharper}. We complement these results by giving a direct
expectation bound for maxima (avoiding an additional norm-to-moment factor) and
by extending Dudley-type chaining bounds to heavy-tailed sub-Weibull processes.

\paragraph{Information-theoretic generalization and RL.}
Information-theoretic bounds date back to \citet{Russo20} and
\citet{xu2017}. Variants based on $f$-divergences and related divergences
have been developed in \citep{jiao2017dependence,hellstrom2025generalization,bu2020}
and references therein.  Most existing information-theoretic bounds rely on
KL-based mutual information (or MGF-based change-of-measure arguments) and are
therefore tailored to bounded or sub-Gaussian losses. We extend this
toolkit to sub-Weibull tails by combining a tail-adaptive divergence and a
multiscale information-theoretic chaining analysis. \cite{behnamnia2025log} proposed a log-sum-exponential estimator that avoids the pitfalls of heavy-tailed rewards. \cite{sun2026curverl} introduced a prompt reweighting method for LLM reasoning based on a distribution‑aware utility functional.

\subsection*{Overview}
Section~\ref{sec:ep} develops the sub-Weibull maximal and Dudley-type
empirical-process tools. Section~\ref{sec:pre} introduces the shifted-log
divergence, the decorrelation lemma, and the corresponding R\'enyi comparison
bounds. Section~\ref{sec:gen_bounds} combines these ingredients to obtain
tail-adaptive selection bounds and a multiscale information-theoretic Dudley
inequality. Section~\ref{sec:rlhf} applies the framework to RLHF under
heavy-tailed rewards, including R\'enyi-regularized and best-of-$n$ policies.
Section~\ref{sec:experiments} provides numerical studies, including controlled
heavy-tail perturbations and token-space reward attacks. Additional proofs and
robustness checks are deferred to the appendix. Table~\ref{tab:roadmap}
summarizes the active results.


\begin{table}[t]
\centering
\footnotesize
\setlength{\tabcolsep}{4pt}
\renewcommand{\arraystretch}{1.05}
\begin{tabular}{@{}>{\raggedright\arraybackslash}p{0.20\linewidth}
                >{\raggedright\arraybackslash}p{0.30\linewidth}
                >{\raggedright\arraybackslash}p{0.48\linewidth}@{}}
\toprule
\textbf{Object} & \textbf{Classical} & \textbf{This paper} \\
\midrule
Finite maxima &
$\E\max_i X_i\lesssim\sqrt{\log n}$ &
$\E\max_i X_i\lesssim(\log n)^{1/\theta}$ (Lemma~\ref{pp-MaximalSW}) \\
Process suprema &
Dudley: $\int\sqrt{\log N_\varepsilon}\,d\varepsilon$ &
Sub-Weibull Dudley: $\int(\log N_\varepsilon)^{1/\theta}\,d\varepsilon$ (Theorem~\ref{Dudley}) \\
Data-dependent selection &
KL mutual information $\sqrt{I(W;S)}$ &
Shifted-log $I_{f_\theta}(W;S)$ and R\'enyi refinements (Theorem~\ref{infob}) \\
Multiscale selection &
Sub-Gaussian information chaining &
Sub-Weibull chaining with $I_{f_\theta}([W]_k;S)$ and R\'enyi refinements (Theorem~\ref{chaining}) \\
RLHF trust regions &
KL trust regions can fail under heavy tails &
R\'enyi trust regions control reward inflation (Theorem~\ref{thm:rlhf}) \\
Best-of-$n$ alignment &
Light-tailed maximal control &
R\'enyi divergence gives $(\log n)^{1/\theta}$-type reward control (Theorem~\ref{renyiregret}) \\
\bottomrule
\end{tabular}
\caption{Roadmap of results.$\theta$ is the sub-Weibull tail parameter and $N_\varepsilon=N(T,d,\varepsilon)$.}
\label{tab:roadmap}
\end{table}

\subsection*{Notation}
For a positive integer $n$, we write $[n]=\{1,\ldots,n\}$. We use $\E$ for
expectation and $\mathbb{P}$ for probability. We denote by $\DD_{\mathrm{KL}}$
the KL divergence and by $\DD_\alpha$ the R\'enyi divergence of order $\alpha>1$.
We write $\|X\|_{\psi_\theta}$ for the Orlicz norm associated with
$\psi_\theta(x)=e^{x^\theta}-1$.  We use $\lesssim$ to hide universal positive
constants. Let $(T, d)$ be a metric space and let $K \subset T$.  A subset $\mathcal{N} \subset T$ is called an \textit{$\varepsilon$-net} for $K$ if, for every $x \in K$, 
there exists $y \in \mathcal{N}$ such that $d(x, y) \leq \varepsilon$. The {covering number}, denoted by $N(K, d, \varepsilon)$,  is the minimum cardinality of such an $\varepsilon$-net. Then $e(T):=\inf\{\varepsilon>0: N(T,d,\varepsilon)=1\}$ is the minimum covering radius of $T$. 
Let $\mathcal{X}$ be the range of a random variable $X$. A partition $\mathcal{P}$ of $\mathcal{X}$ is a finite collection of disjoint sets $P_i$ such that \(\bigcup_i P_i = \mathcal{X}\).

\section{Sub-Weibull Process Theory}\label{sec:ep}

This section develops empirical-process tools for controlling maxima and
suprema of sub-Weibull random variables and processes, with an emphasis on the
heavy-tailed regime $0<\theta<1$.  In this regime, MGFs may be infinite, so many classical MGF-based arguments must be replaced
by Orlicz-norm methods.  We establish (i) a sharp maximal inequality for a
finite family of sub-Weibull variables and (ii) a Dudley-type chaining bound
for $(\theta,C)$-sub-Weibull processes.  These are \emph{uniform} (worst-case)
bounds; Section~\ref{sec:gen_bounds} will refine them for data-adaptive
algorithms using information measures.

\subsection{A sharp maximal inequality}

\begin{lemma}[Maximal inequality for heavy-tailed sub-Weibull variables]\label{pp-MaximalSW}
Let $0<\theta<1$ and $\psi_\theta(x)=e^{x^\theta}-1$. Let $\{X_i\}_{i=1}^n$ be random variables with
\(
\max_{i\in [n]}\|X_i\|_{\psi_\theta}<\infty.
\)
Then
\begin{equation}\label{eq:MaximalSW}
\mathbb{E}\Bigl[\max_{i\in [n]}|X_i|\Bigr]
\;\le\;
\psi_\theta^{-1}\bigl(\psi_\theta(x_\theta)+n\bigr)\,
\max_{i\in [n]}\|X_i\|_{\psi_\theta}
\;=\;O\bigl((\log n)^{1/\theta}\bigr),
\end{equation}
where $
\psi_\theta^{-1}\bigl(\psi_\theta(x_\theta)+n\bigr)
=\bigl[\log\bigl(1+\psi_\theta(x_\theta)+n\bigr)\bigr]^{1/\theta}$ with $x_\theta:=\Bigl(\frac{1-\theta}{\theta}\Bigr)^{1/\theta}$.
\end{lemma}

The $1/\theta$ in~\eqref{eq:MaximalSW} is sharp: for i.i.d.
$\mathrm{Weibull}(\theta)$ variables,
$\mathbb{E}[\max_{i\in [n]}X_i]\asymp (\log n)^{1/\theta}$; see Appendix~\ref{app:sec2}.
Thus, heavier tails (smaller $\theta$) necessarily worsen the growth rate of the
maximum.  A useful heuristic is that sub-Weibull variables behave like power
transforms of unbounded sub-Gaussian variables, so decreasing $\theta$ makes
rare extremes more influential.

A common route is to first control the Orlicz norm of the maximum and then
convert that Orlicz bound to an expectation bound via a norm-to-moment
inequality \citep{mendes2023concentration, zhang2022sharper}.  This often
introduces an additional factor $\Gamma(1/\theta+1)$, which grows rapidly as
$\theta\downarrow0$ and can noticeably loosen prefactors in the heavy-tailed
regime.  Lemma~\ref{pp-MaximalSW} bounds the expectation directly and avoids this
extra loss while retaining the optimal $O\bigl((\log n)^{1/\theta}\bigr)$
scaling; a quantitative comparison is given in Appendix~\ref{app:max_comparison}.

\subsection{From maxima to suprema: sub-Weibull processes}

Passing from finite maxima to suprema over a general index set raises
measurability issues: $\omega\mapsto\sup_{t\in T}X_t(\omega)$ need not be
measurable.  A standard way to avoid outer expectations is to assume
\emph{separability}, which reduces the supremum to one over a countable dense
subset \citep[Definition~5.22]{vanhandle2016probability}.

\begin{definition}[Separable process]\label{def:separable}
Let $(T,d)$ be a metric space. A stochastic process $\{X_t\}_{t\in T}$ is
\emph{separable} if there exists a countable set $T_0\subseteq T$ such that
\[
X_t \in \lim_{s \to t,~s \in T_0} X_s\;\;\text{for all } t \in T\;\text{a.s.}
\]
Equivalently, for every $t\in T$ there exists $\{s_k\}\subseteq T_0$ with
$s_k\to t$ and $X_{s_k}\to X_t$ a.s..
\end{definition}

Empirical process theory typically controls oscillations of a process via tail
bounds on increments relative to a metric.  In the sub-Gaussian case, an
increment condition of the form
$\|X_t-X_s\|_{\psi_2}\lesssim d(t,s)$ leads to Dudley's entropy bound
\(\E\sup_{t\in T}X_t\lesssim\int_0^{\infty}\sqrt{\log N(T,d,\varepsilon)}\,d\varepsilon\).
In heavy-tailed settings, MGFs may fail to exist, yet increments are still often
controlled by an exponential-of-a-power tail, motivating the following
definition.

\begin{definition}[Sub-Weibull process]\label{subWp}
Let $\theta>0$, $C>0$, and let $(T,d)$ be a metric space. A mean-zero process
$\{X_t\}_{t\in T}$ is a \emph{$(\theta,C)$-sub-Weibull process} (with respect to
$d$) if
\[
\mathbb{E}\exp\!\left(\frac{|X_t-X_s|^\theta}{[Cd(t,s)]^\theta}\right)\le 2,
\qquad \forall s,t\in T~\text{s.t.}~s\ne t.
\]
Equivalently, $\|X_t-X_s\|_{\psi_\theta}\le C\,d(t,s)$ for all $s,t\in T$.
\end{definition}

For $\theta\ge 1$, such increment control implies suitable log-MGF bounds and
classical chaining machinery applies.  When $0<\theta<1$, MGF-based arguments
generally break down; nevertheless, Orlicz increment control remains sufficient
to develop a Dudley-type entropy bound once we can control maxima at each scale
(Lemma~\ref{pp-MaximalSW}).

\begin{theorem}[Heavy-tailed Dudley inequality]\label{Dudley}
Let $(T,d)$ be a finite metric space and let $\{X_t\}_{t\in T}$ be a mean-zero
$(\theta,C)$-sub-Weibull process with $0<\theta<1$. Then
\begin{equation}\label{eq:heavy_dudley}
\mathbb{E}\Bigl[\sup_{t\in T}X_t\Bigr]
\le 4C K_\theta\int_0^\infty \bigl[\log N(T,d,\varepsilon)\bigr]^{1/\theta}\,d\varepsilon,
\end{equation}
where $N(T,d,\varepsilon)$ is the covering number and
\[
K_\theta
:=\sup_{x\ge 2}\left(\frac{\log(1+\psi_\theta(x_\theta)+x)}{\log x}\right)^{1/\theta}
<\infty,
\qquad x_\theta=\left(\frac{1-\theta}{\theta}\right)^{1/\theta}.
\]
\end{theorem}

Theorem~\ref{Dudley} is stated for finite $T$ to keep discretization explicit.
For general index sets, separability allows one to pass to a countable dense
subset, yielding the following corollary (proved in Appendix~\ref{se:Dudley1}).

\begin{corollary}\label{Dudley1}
Let $\{X_t\}_{t\in T}$ be a separable, mean-zero $(\theta,C)$-sub-Weibull process
on $(T,d)$. If we assume that $\gamma<\infty$, then
\[
\mathbb{E}\Bigl[\sup_{t\in T}X_t\Bigr]
\le 4 C K_\theta \int_0^\infty \bigl[\log N(T,d,\varepsilon)\bigr]^{1/\theta}\,d\varepsilon
=: \gamma.
\]
\end{corollary}

The finiteness of $\gamma$ implies $\E[\sup_{t\in T}X_t]<\infty$.  When
$\theta=2$ (sub-Gaussian),~\eqref{eq:heavy_dudley} reduces to Dudley's classical
entropy integral; $\theta=1$ matches the usual sub-exponential analogue.  For
$0<\theta<1$, the exponent $1/\theta>1$ inflates the entropy integrand,
quantitatively capturing the degradation in supremum control induced by heavier
tails.

\subsection{Maximal and Dudley bounds are insufficient for data-adaptive algorithms}

Lemma~\ref{pp-MaximalSW} and Theorem~\ref{Dudley} control maxima and suprema via
\emph{global} complexity terms (e.g., $\log n$ or an entropy integral).  This
worst-case nature is intrinsic: replacing a data-dependent choice by a supremum
over all candidates necessarily pays for the size/entropy of the ambient class.
In modern learning problems, however, algorithms are typically data-adaptive and
probe only a localized, sample-dependent region of the hypothesis space.
Consequently, purely uniform maximal inequalities can be overly conservative and
may become vacuous when used as surrogates for generalization control.

This motivates the information-theoretic approach developed next: rather than
bounding a random, data-dependent choice by a global supremum, we control it in
terms of how much \emph{information} the algorithm extracts from the data.
Section~\ref{sec:gen_bounds} develops tail-adaptive information bounds and a
multiscale (chaining) refinement; Section~\ref{sec:experiments} provides a
numerical illustration of the gap between uniform and information-theoretic
chaining bounds.

\section{Preliminaries on \texorpdfstring{$f$}{f}-Divergence and \Renyi Divergence}\label{sec:pre}

This section collects divergence notions and change-of-measure tools used in the
heavy-tailed analysis.  Classical information-theoretic generalization proofs
typically combine (i) an MGF bound for the loss and (ii) a KL-based variational
representation (Donsker--Varadhan).  When MGFs do not exist, this route breaks
down.  We instead use a \emph{shifted-log} $f_\theta$-divergence tailored to
sub-Weibull tails and compare it to R\'enyi divergence.  The key payoff is a
\emph{decorrelation} lemma: it bounds expectations under a dependent coupling
(e.g., $W$ depends on $S$) by a divergence term plus a decoupled moment term.

\subsection{\texorpdfstring{$f$}{f}-divergence and \Renyi divergence}

\begin{definition}[$f$-divergence and \Renyi mutual information]\label{f-divergence}
Let $P$ and $Q$ be probability measures on a measurable space $(\mathcal{X},\mathcal{F})$, such that $P\ll Q$.
For a convex function $f:(0,\infty)\to\mathbb R$ with $f(1)=0$, the $f$-divergence
is defined as
\begin{equation}\label{eq:f-divergence}
\DD_f(P\|Q)
\;:=\;
\begin{cases}
\displaystyle \int_{\mathcal X} f\left(\frac{dP}{dQ}\right)dQ, & \text{if } P\ll Q,\\
\infty, & \text{otherwise.}
\end{cases}
\end{equation}
For \( \alpha > 0 \) and \( \alpha \neq 1 \), the \emph{\Renyi divergence} of order \( \alpha \) from \( P \) to \( Q \) is defined as
\[
\DD_\alpha(P \| Q) = \frac{1}{\alpha - 1} \log \int_{\mathcal{X}} \left( \frac{dP}{dQ}\right)^\alpha dQ=\frac{1}{\alpha - 1} \log \int_{\mathcal{X}} \left( \frac{dP}{dQ}\right)^{\alpha-1} dP.
\]
\end{definition}

Important special cases of $f$-divergence include KL divergence ($f(x)=x\log x$). The \Renyi divergence is a generalization of the KL divergence, defined below. 
When $\alpha=2$, $\DD_2(P\|Q)=\log\bigl(1+\chi^2(P\|Q)\bigr)$, where $\chi^2(P\|Q)
=
\int_{\mathcal X}
\left({dP}/{dQ}\right)^2dQ - 1$ is the chi-square divergence for $P\ll Q$. As
$\alpha\to\infty$ one obtains
$\DD_\infty(P\|Q)=\log\|dP/dQ\|_\infty$. 
\begin{lemma}[Basic properties of \Renyi divergence, \citet{RenyiKL}]\label{Renyi}
For any $P,Q$ on the same measurable space, $\DD_\alpha(P\|Q)$ is nondecreasing in
$\alpha$. Moreover,
\begin{enumerate}
\item[(i)] for $0<\alpha_1\le \alpha_2$, $\DD_{\alpha_1}(P\|Q)\le \DD_{\alpha_2}(P\|Q)$;
\item[(ii)] $\lim_{\alpha\to 1}\DD_{\alpha}(P\|Q)=\DD_{\mathrm{KL}}(P\|Q)$.
\end{enumerate}
\end{lemma}





\begin{definition}[Mutual information]\label{renyi_Mutual_information}
Let $(W,S)$ be a pair of random variables with joint distribution $P_{W,S}$.
For $\alpha>1$, the \emph{R\'enyi mutual information} is
\begin{equation}\label{eq:Renyi_MI}
I_\alpha(W;S):=\DD_\alpha(P_{W,S}\|P_W\otimes P_S).
\end{equation}
\end{definition}

\subsection{Shifted-log divergences and \Renyi upper bounds}

To obtain tail-adaptive bounds under sub-Weibull index $\theta$, we will use the
following family of convex functions with additive shift $A$,
\begin{equation}\label{eq:f_theta}
f_\theta(x)=x\log^{\frac{1}{\theta}}(x+A), \qquad \theta>0,\ A\ge 1.
\end{equation}
The $A$ prevents singular behavior near $0$ and the power $\theta^{-1}$ matches the power in the optimal scaling $O\bigl((\log n)^{1/\theta}\bigr)$ for sub-Weibull maximum inequality (Lemma \ref{pp-MaximalSW}) in the following heavy-tailed decorrelation arguments.
Strictly speaking, $f_\theta(1)=\log^{1/\theta}(1+A)$, so this is an
unnormalized shifted-log $f$-divergence rather than a normalized $f$-divergence
generator vanishing at independence.  We keep the notation $\DD_{f_\theta}$ and
$I_{f_\theta}$ below for the corresponding shifted-log dependence functionals;
the comparison and data-processing arguments are used only up to this fixed
baseline, and none of the results relies on the functional being zero under
independence.

\subsubsection{\Renyi divergence-based upper bounds}

\begin{lemma}\label{lem:key1}
Let \(P\ll Q\) be probability measures on a measurable space \((\mathcal X,\mathcal F)\), and suppose that
\(\DD_\alpha(P\|Q)<\infty\) for some \(\alpha>1\). Fix \(\theta>0\) and define
\(f_\theta(x):=x\log^{1/\theta}(x+A)\).
Assume that \(A\) satisfies $A\geq 1$ when $\theta\geq1$ and when $\theta<1$
\[
A \ge \exp\!\left(\frac{1}{\alpha-1}\Bigl(\frac{1}{\theta}-1\Bigr)\right)\quad \text{for }1<\alpha\le 2,
\qquad\text{and}\qquad
A \ge \exp\!\left(\frac{1}{\theta}-1\right)\quad \text{for }\alpha>2.
\]
Then
\begin{equation}\label{eq:DF1}
\DD_{f_\theta}(P\|Q) \le [\DD_\alpha(P\|Q)+C_{\alpha,\theta}]^{1/\theta},
\end{equation}
where, for \(1<\alpha\le 2\),
\[
C_{\alpha,\theta}
:=
\begin{cases}
\displaystyle \frac{1}{\alpha-1}\log\bigl(1+A^{\alpha-1}\bigr), & 0<\theta<1,\\[0.9em]
1+(A-1)^{1/\theta}, & \theta\ge 1,
\end{cases}
\qquad\text{and}\qquad
C_{\alpha,\theta}=C_{2,\theta}\ \text{for }\alpha>2.
\]
\end{lemma}

Lemma~\ref{lem:key1} provides a clean {power-type} comparison:
$\DD_{f_\theta}$ is controlled by a $1/\theta$ power of R\'enyi divergence up to
an additive constant inside the power. We use this comparison throughout the
R\'enyi refinements below.

\subsection{Decorrelation inequalities}\label{sub:decorr}

We now state the main change-of-measure tool used throughout the paper.
Given a joint law $\mu$ (e.g., of $(W,S)$) and a product reference law $\nu$
(e.g., $P_W\otimes P_S$), we want to control $\E_\mu[r]$ by a divergence between
$\mu$ and $\nu$ plus a moment term under $\nu$.  In the sub-Gaussian case, this
is commonly done via the Donsker--Varadhan representation of KL divergence.
Here we use a Young-type inequality based on $f_\theta$, which does not require
MGFs.

\begin{lemma}[Decorrelation lemma]\label{de}
Let $\theta>0$ and $A\ge (2^{\lceil \frac{2}{\theta} \rceil-2} \lceil \frac{2}{\theta} \rceil !)^2 \vee 1$. Let $\mu$ and $\nu$ be probability measures on a
measurable space $(\Omega,\mathcal F)$. Let $r\ge 0$ be measurable and assume
$r\in L^1(\mu)$ and $\exp(r^\theta)\in L^1(\nu)$. Let $f_\theta(x)=x\log^{1/\theta}(x+A)$, then
\begin{equation}\label{eq:dec}
\mathbb E_\mu r
\le 2^{1/\theta}\DD_{f_\theta}(\mu\|\nu)+\mathbb{E}_\nu\exp(r^\theta).
\end{equation}
\end{lemma}

The Lemma~\ref{de} will be applied with $\mu=P_{W,S}$ and $\nu=P_W\otimes P_S$ and
with $r$ equal to the data or the reward in RLHF in Section \ref{sec:rlhf}; the first term captures the dependence
between $W$ and $S$, while the second term is a sub-Weibull moment under the
product measure.  The proof of Lemma~\ref{de} is based on a new Young-type
inequality and is deferred to Appendix~\ref{proofse3}.



\section{\Renyi information refinements of maximal and Dudley-type bounds}\label{sec:gen_bounds}

This section turns the sub-Weibull empirical-process tools from
Section~\ref{sec:ep} into information-theoretic bounds for data-dependent
selection. The key mechanism is the shifted-log divergence from
Section~\ref{sec:pre}: it matches sub-Weibull tails through the decorrelation
lemmas and can be upper bounded by R\'enyi mutual information via
Lemma~\ref{lem:key1}.

We begin with the selection problem of bounding $\E[X_W]$ when $W$ is a
data-dependent index.  In the sub-Gaussian setting, KL mutual information gives
such a bound \citep{Russo20}, but Section~\ref{se:whyKLfails} shows that this
route can fail under heavy tails.  Theorem~\ref{infob} gives the shifted-log and R\'enyi replacements, while Theorem~\ref{chaining} lifts the
single-index result to a multiscale Dudley-type bound for data-dependent
indices in metric spaces.



\subsection{Why KL mutual information fails under heavy tails}\label{se:whyKLfails}

Let $S=\{X_i\}_{i=1}^n$ and let $W=W(S)$ be a (possibly randomized) data-dependent
index taking values in $[n]$.  When the $X_i$ are zero-mean i.i.d.\ sub-Gaussian
with variance proxy $\sigma^2$, \citet{Russo20} showed that
\begin{equation}\label{eq:Russo}
\mathbb{E}[X_W] \le \sqrt{2\sigma^2I(W;S)},
\end{equation}
where $I(\cdot;\cdot)$ is KL-based mutual information.  In particular,
choosing $W=\arg\max_{i\in[n]}X_i$ and using $I(W;S)\le \log n$ recovers the
classical maximal inequality
\(
\mathbb{E}[\max_{i\in[n]} X_i]\lesssim \sqrt{\log n}.
\)

In heavy-tailed regimes, KL mutual information may no longer control
$\mathbb{E}[X_W]$.  The following example constructs a randomized maximum
selector $W$ for which $I(W;S)=O(1)$ while $\mathbb{E}[X_W]$ diverges with $n$.

\begin{example}\label{ex:weibull}
Let $S=\{X_i\}_{i=1}^n$ be i.i.d.\ $\mathrm{Weibull}(\theta)$ random variables
with $\theta\in(0,1)$, i.e.\ $\mathbb{P}(X_1\ge x)=\exp(-x^\theta)$ for $x>0$.
Define a randomized selector $W$ by
\[
W=
\begin{cases}
\arg\max_{i\in[n]}X_i, & \text{with probability }\varepsilon,\\
U, & \text{with probability }1-\varepsilon,
\end{cases}
\]
where $U$ is uniform on $[n]$ and independent of $S$.
A direct computation gives
\[
\begin{aligned}
I(W;S)
= \DD_{\mathrm{KL}}(P_{W,S}\,\|\,P_W\otimes P_S) = \frac{n-1}{n}(1-\varepsilon)\log(1-\varepsilon)
   + \frac{(n-1)\varepsilon+1}{n}\log\bigl((n-1)\varepsilon+1\bigr).
\end{aligned}
\]

Choosing $\varepsilon=c/\log n$ for any fixed $c>0$ yields $I(W;S)=O(1)$ as
$n\to\infty$.  On the other hand,
$\mathbb{E}[\max_{i\in[n]}X_i]\asymp (\log n)^{1/\theta}$ (see Appendix~\ref{app:sec2}), and therefore
\[
\mathbb{E}[X_W]
\;\ge\;
\varepsilon\,\mathbb{E}[\max_{i\in[n]}X_i]
=\Omega \bigl((\log n)^{1/\theta-1}\bigr).
\]
Thus, for heavy-tailed Weibull data, the selection bias $\mathbb{E}[X_W]$ can
diverge while the KL mutual information remains bounded.
\end{example}

\subsection{Why \Renyi mutual information and shifted-log \texorpdfstring{$f$}{f}-divergences}

Example~\ref{ex:weibull} reflects a general phenomenon: when MGFs do not exist,
KL-based information can be too weak to control the contribution of rare but
extreme events.  To obtain a tail-adaptive substitute, we applied the proposed shifted-log
family \eqref{eq:f_theta} and use the associated $f_\theta$-mutual information $I_{f_\theta}(W;S)$.
It matches sub-Weibull tails and, via Lemma~\ref{lem:key1}, admits explicit upper bounds in terms of R\'enyi mutual
information.

\begin{theorem}[Tail-adaptive selection bound]\label{infob}
Let $S=\{X_i\}_{i=1}^n$ with $X_i \sim \subw(\theta)$ for $\theta>0$, and let $W=W(S)$ be any $[n]$-valued, possibly randomized, selector. Let $I_{f_\theta}(W;S)$ be the $f_\theta$-mutual information associated with $f_\theta(x)=x\log^{1/\theta}(x+A)$, where $A$ satisfies the condition in Lemma~\ref{de}. Then
\[
\mathbb{E}|X_W|
\;\le\;
\max_{i\in[n]}\|X_i\|_{\psi_\theta}\,\Bigl(2^{1/\theta}I_{f_\theta}(W;S)+2\Bigr).
\]
Moreover, if $A$ satisfies the conditions of Lemma~\ref{lem:key1}, then
\[
\mathbb{E}|X_W|
\;\le\;
\max_{i\in[n]}\|X_i\|_{\psi_\theta}\,
\Bigl(2^{1/\theta}\bigl(I_\alpha(W;S)+C_{\alpha,\theta}\bigr)^{1/\theta}+2\Bigr),
\]
where $I_\alpha(W;S)$ is the R\'enyi-$\alpha$ mutual information and
$C_{\alpha,\theta}$ depend only on $(\alpha,\theta)$.
\end{theorem}

When $\theta=2$ and $\alpha\to 1$, Theorem~\ref{infob} is similar to Russo and Zou's bound
\eqref{eq:Russo} up to universal constants.  Taking $W=\arg\max_{i\in[n]}|X_i|$
also recovers the optimal maximal growth rate
$\mathbb{E}[\max_i|X_i|]\lesssim(\log n)^{1/\theta}$; see Example~\ref{Max} in
Appendix~\ref{se: proof2}.

\subsection{A Dudley-type bound via multiscale information}

Theorem~\ref{infob} bounds a \emph{single} data-dependent selection.  To obtain a
Dudley-type bound for a data-dependent index in a general metric space, we
combine Theorem~\ref{infob} with a multiscale chaining construction, following
\citet{asadi2018chaining,hellstrom2025generalization}.

\begin{definition}[$\varepsilon$-partitions and increasing sequences]\label{def:part}
Let $(\mathcal{W},d)$ be a metric space. A partition $\mathcal{P}=\{A_1,\ldots,A_m\}$ of $\mathcal{W}$ is called an \emph{$\varepsilon$-partition} if for every $i\in[m]$ there exists $w_i\in\mathcal{W}$ such that
\[
A_i \subseteq \mathcal{B}_d(w_i,\varepsilon),
\qquad
\mathcal{B}_d(w_i,\varepsilon):=\{w\in\mathcal{W}: d(w,w_i)\le \varepsilon\}.
\]
A sequence of partitions $\{\mathcal{P}_k\}_{k=k'}^{\infty}$ is called \emph{increasing} if for all $k\ge k'$ and every $A\in\mathcal{P}_{k+1}$ there exists $B\in\mathcal{P}_k$ such that $A\subseteq B$. For each $w\in\mathcal{W}$ and $k\ge k'$, let $[w]_k$ denote the unique set $A\in\mathcal{P}_k$ containing $w$.
\end{definition}


\begin{theorem}[Information-theoretic Dudley bound under sub-Weibull tails]\label{chaining}
 Let $(T,d)$ be totally bounded and let $X_t$ be separable
$(\theta,C)$-sub-Weibull process with a.s. uniformly continuous paths, and with finite
Dudley integral in Corollary~\ref{Dudley1}. Let $W=W(S)$ be a $T$-valued random
index (equivalently, $W$ is measurable with respect to $X$). Let
$\{\mathcal{P}_k\}_{k\ge 0}$ be an increasing sequence of partitions such that
$\mathcal{P}_0=\{T\}$ and $\mathcal{P}_k$ is an $e(T)2^{-k}$-partition for $k\ge 1$.
For each $A\in\mathcal{P}_k$, choose $t_A\in T$ so that
$A\subset \mathcal{B}_d\bigl(t_A,e(T)2^{-k}\bigr)$, and if $A\subset B\in \mathcal{P}_{k-1}$, then $t_A\in B$. Let $[W]_k\in\mathcal P_k$ be the cell-valued variable containing $W$. Assume the shift in $f_\theta$ satisfies the condition in Lemma~\ref{de}. Then
\[
\mathbb{E}[X_W]
\;\le\;
Ce(T)\sum_{k=1}^\infty 2^{-(k-1)}
\Bigl(2^{1/\theta}I_{f_\theta}([W]_k;S)+2\Bigr).
\]
Assume $A$ satisfies the conditions of Lemma~\ref{lem:key1}. Then, for any
$\alpha>1$ such that $I_\alpha([W]_k;S)<\infty$ for all $k$ and $C_{\alpha,\theta}$ in Lemma~\ref{lem:key1},
\begin{equation}\label{computmeth3}
\mathbb{E}[X_W]
\;\le\;
Ce(T)\sum_{k=1}^\infty 2^{-(k-1)}
\Bigl(
2^{1/\theta}\bigl(I_\alpha([W]_k;S)+C_{\alpha,\theta}\bigr)^{1/\theta}
+2
\Bigr).
\end{equation}
\end{theorem}

When $\theta=2$, Theorem~\ref{chaining} is similar to the mutual-information chaining
bound of \citet[Theorem 4]{asadi2018chaining} up to universal constants.  When
$\theta=1$, it yields the information-theoretic chaining bound under
sub-exponential tails.  For $\theta\in(0,1)$, the exponent $1/\theta>1$ inflates
the per-scale complexity term, quantifying the deterioration caused by heavier
tails.  The proof uses only Orlicz-norm control of increments and multiscale
information, avoiding MGFs and circumventing the KL-based failure exhibited in
Example~\ref{ex:weibull}.

\section{Applications to RLHF with Heavy-Tailed Rewards}\label{sec:rlhf}

This section studies our heavy-tailed information-theoretic bounds: safety alignment of large language models (LLMs) under heavy-tailed reward functions. Recent works on statistical principles can directly design  trustworthy alignment methods \citep{su2026large}, and we focus more on robust estimation principles.


We study reinforcement learning from human feedback (RLHF) for aligning LLMs.
 Let $\pi(y\mid x)$ denote a policy (a Markov kernel from prompts to responses),
where $x\in\mathcal X$ is a prompt and $y\in\mathcal Y$ is a generated response.
Prompts are drawn from a distribution $\rho_X$ on $\mathcal X$, and
$\pi(\cdot\mid x)$ specifies the conditional distribution of responses.
In this sense, RLHF can be viewed as a contextual bandit problem.

Given a reference policy $\pi_0(\cdot\mid x)$ (typically a pretrained model),
the goal is to construct a new policy $\pi$ that achieves higher expected reward
$r(x,y)$.  In practice, $r$ is learned from preference data and can be
misspecified; we refer to the true human-aligned objective as the \emph{gold
reward} and the learned reward model as the \emph{proxy reward}.

\subsection{KL-constrained and best-of-\texorpdfstring{$n$}{n} alignment}

A standard formulation of RLHF is the KL-constrained optimization
\citep{ouyang2022training}:
\begin{equation}\label{opt}
\max_{\pi}\;
\mathbb{E}_{x\sim\rho_X,\;y\sim\pi(\cdot\mid x)}[r(x,y)]
\quad\text{s.t.}\quad
\DD_{\mathrm{KL}}\bigl(\pi(\cdot\mid x)\,\|\,\pi_0(\cdot\mid x)\bigr)\le\epsilon
\quad\forall x,
\end{equation}
where $\epsilon>0$ is an information budget.
By Donsker--Varadhan representation, the optimizer is an exponentially tilted distribution \citep{yang2024asymptotics} with a Lagrange multiplier $\lambda>0$
\begin{equation}\label{KLsol}
\pi^{*}_{\mathrm{KL}}(y\mid x)
=\frac{\pi_0(y\mid x)\exp(r(x,y)/\lambda)}
{\mathbb{E}_{Y\sim\pi_0(\cdot\mid x)}\exp(r(x,Y)/\lambda)}.
\end{equation}

An alternative inference-time strategy is the \emph{best-of-$n$} policy
\citep{stiennon2020learning,nakano2021webgpt}.
Given a prompt $x$, one samples $n$ i.i.d.\ responses
$Y_1,\dots,Y_n\sim\pi_0(\cdot\mid x)$ and selects
$Y_n^*\in \arg\max_{i\in [n]} r(x,Y_i)$. The resulting policy
$\pi_n(\cdot\mid x)$ is the distribution of $Y_n^*$.
Under bounded rewards and regularity assumptions, best-of-$n$ policies are
asymptotically close to the KL-constrained optimizer
\citep{mroueh2025information,yang2024asymptotics}.

\subsection{Catastrophic Goodhart and \Renyi-regularized RLHF}

Under light-tailed assumptions, one can derive reward guarantees using KL-based
arguments and data processing \citep{mroueh2025information}.
When the proxy reward is heavy-tailed, however, KL regularization can fail
catastrophically: the proxy reward can diverge even when the KL constraint is
small, a phenomenon known as \emph{catastrophic Goodhart}
\citep{kwa2024catastrophic}.
Formally, for any heavy-tailed distribution $Q$ and any $\epsilon>0$, there exist
distributions $P$ with \emph{arbitrarily large mean} such that
$\DD_{\mathrm{KL}}(P\|Q)\le\epsilon$.
Intuitively, exponential tilting \eqref{KLsol} can place disproportionate mass
on extreme-reward events, while the KL penalty---an average discrepancy under the
candidate policy---need not scale proportionally with such rare tail shifts.

To mitigate this, we consider RLHF with R\'enyi-$\alpha$ regularization:
\begin{equation}\label{optRENYI}
\max_{\pi}\;
\mathbb{E}_{x\sim\rho_X,\;y\sim\pi(\cdot\mid x)}[r(x,y)]
\quad\text{s.t.}\quad
\DD_{\alpha}\bigl(\pi(\cdot\mid x)\,\|\,\pi_0(\cdot\mid x)\bigr)\le\epsilon
\quad\forall x,
\end{equation}
for some $\alpha>1$.
Compared to KL, R\'enyi penalizes high density ratios more
aggressively through a power $\alpha$, which induces a qualitatively different
structure in the optimizer.

\begin{lemma}[Solution to R\'enyi-constrained optimization]\label{renyisol}
Let $\alpha>1$ and $\epsilon>0$. Fix $x\in\mathcal X$ and write $P_0:=\pi_0(\cdot\mid x), r_x(y):=r(x,y)$. Assume that policies in \eqref{optRENYI} are optimized over probability measures $\ll P_0$, and assume
$r_x\in L^{\frac{\alpha}{\alpha-1}}(P_0)$.
Let
\[r_x^\star:=\operatorname*{ess\,sup}_{y\sim P_0} r_x(y),
\qquad
M_x:=\{y\in\mathcal Y:r_x(y)=r_x^\star\},
\qquad
m_x:=P_0(M_x),
\]
with the convention $-\log 0=+\infty$.
Then the optimizer of \eqref{optRENYI} satisfies
\begin{enumerate}
\item If $\epsilon\ge -\log m_x$, then the R\'enyi ball is large enough to concentrate all probability mass on
the maximal-reward set. One has an optimal policy
$
\pi_M(dy\mid x)
=
\frac{\mathbf 1_{M_x}(y)}{m_x}\,\pi_0(dy\mid x)$.

\item If $\epsilon<-\log m_x$, then the R\'enyi constraint is active at the optimizer, i.e.
$
\mathcal D_\alpha\bigl(\pi^*(\cdot\mid x)\|\pi_0(\cdot\mid x)\bigr)
=
\epsilon$. Moreover, there exists a threshold $t\in\mathbb R$ such that
\[
\frac{d\pi^*(\cdot\mid x)}{d\pi_0(\cdot\mid x)}(y)
=
\frac{(r(x,y)-t)_+^{\frac1{\alpha-1}}}
{\mathbb E_{Y\sim\pi_0(\cdot\mid x)}
[(r(x,Y)-t)_+^{\frac1{\alpha-1}}]}.
\]
Equivalently, when densities exist,
$\pi^*(y\mid x)
=
\frac{
\pi_0(y\mid x)(r(x,y)-t)_+^{\frac1{\alpha-1}}
}{
\mathbb E_{Y\sim\pi_0(\cdot\mid x)}
[(r(x,Y)-t)_+^{\frac1{\alpha-1}}]
}$. The threshold $t$ is chosen so that the R\'enyi constraint is met with equality.
\end{enumerate}
\end{lemma}

Note that the first case is degenerate and rarely occurs in continuous problems. For instance, if $r(x,\cdot)$ is continuous and has a unique maximizer under a non-atomic reference policy $\pi_0(\cdot\mid x)$, then $M_x$ is a singleton and hence $\pi_0(M_x\mid x)=0$, so only the second case applies for any finite $\epsilon$.

Lemma~\ref{renyisol} highlights the key structural difference from KL: R\'enyi
regularization yields a \emph{truncated power-law} reweighting of the reference
policy, whereas KL yields exponential tilting.  The truncation limits the
influence of extreme-reward tail events, which is precisely the regime where KL
can be vulnerable.
The special case $\alpha=2$ corresponds to $\chi^2$-preference optimization, for
which $\pi^*(y\mid x)\propto\pi_0(y\mid x)(r(x,y)-t)_+$; see
\cite{huang2025correcting}.

\begin{theorem}[Reward guarantees for RLHF with R\'enyi regularization]\label{thm:rlhf}
Fix $x\in\mathcal X$ and let $\pi_0(\cdot\mid x)$ be the reference policy.
Abbreviate $\pi(\cdot\mid x)$ and $\pi_0(\cdot\mid x)$ by $\pi$ and $\pi_0$.
Define the centered reward
\[
\bar r(x,y)=r(x,y)-\mathbb{E}_{Y\sim\pi_0}[r(x,Y)].
\]
Assume $\|\bar r(x,Y)\|_{\psi_\theta}\le C$ for some $\theta>0$ and
$Y\sim\pi_0(\cdot\mid x)$.
Then for any policy $\pi(\cdot\mid x)$,
\[
\mathbb{E}_{\pi} r(x,Y) - \mathbb{E}_{\pi_0} r(x,Y)
\;\le\;
C\Bigl(2^{1/\theta}[\DD_\alpha(\pi\|\pi_0)+C_{\alpha,\theta}]^{1/\theta}
+2\Bigr),
\]
where $C_{\alpha,\theta}$ is the constant in Lemma~\ref{lem:key1}.
In particular, for the optimizer $\pi^*$ of~\eqref{optRENYI},
\[
\mathbb{E}_{x\sim\rho_X}\mathbb{E}_{Y\sim\pi^*(\cdot\mid x)} r(x,Y)
\;\le\;
\mathbb{E}_{x\sim\rho_X}\mathbb{E}_{Y\sim\pi_0(\cdot\mid x)} r(x,Y)
+ C\Bigl(2^{1/\theta} [\epsilon+C_{\alpha,\theta}]^{1/\theta}+2\Bigr).
\]
\end{theorem}

Theorem~\ref{thm:rlhf} shows that, 
R\'enyi-constrained RLHF rules out unbounded reward inflation at any fixed (divergence) budget under sub-Weibull reward tails, thereby preventing catastrophic Goodhart in the reward model.

\subsection{Best-of-\texorpdfstring{$n$}{n} policies under heavy-tailed rewards}

We now analyze whether catastrophic Goodhart arises for best-of-$n$ policies.
We adopt the following structural assumption
\citep{beirami2024theoretical,mroueh2025information}.

\begin{assumption}[Reward structure]\label{ass:reward-structure}
For each context $x$, define
\[
r(x)=r(x,Y),\quad Y\sim\pi_0(\cdot\mid x),
\qquad
r_n(x)=\max_{i\in [n]} r(x,Y_i),
\quad Y_i\overset{\mathrm{i.i.d.}}{\sim}\pi_0(\cdot\mid x).
\]
There exists a stochastic map $H_x$ such that
$H_x(r(x))\overset{d}{=}\pi_0(\cdot\mid x)$ and
$H_x(r_n(x))\overset{d}{=}\pi_n(\cdot\mid x)$.
\end{assumption}

Under Assumption~\ref{ass:reward-structure} and the data processing inequality
for R\'enyi divergence (cf.\ \citealp{Yihong25}, p.~120),
\[
\DD_\alpha(\pi_n(\cdot\mid x)\|\pi_0(\cdot\mid x))
\;\le\;
\DD_\alpha(r_n(x)\|r(x)).
\]

\begin{theorem}[Best-of-$n$ under heavy-tailed rewards]\label{renyiregret}
Assume $\|r(x,Y)\|_{\psi_\theta}\le C$ for some $\theta>0$, $\alpha>1$ and
$Y\sim\pi_0(\cdot\mid x)$, and assume Assumption~\ref{ass:reward-structure}.
If the trust-region budget $n\le e^\epsilon$ for $\epsilon>0$, then
\[
\DD_\alpha(\pi_n(\cdot\mid x)\|\pi_0(\cdot\mid x))
\le \frac{1}{\alpha-1}\log\!\left(\frac{n^\alpha}{\alpha(n-1)+1}\right)
\le  \epsilon,
\]
and
\[
\E_{\pi_n}r-\E_{\pi_0}r
\le
C \left( 2^{1/\theta} \left[ \DD_\alpha(\pi_n \| \pi_0) +C_{\alpha,\theta}\right]^{1/\theta} + 2 \right)
\le
C \left( 2^{1/\theta} \left[ \e +C_{\alpha,\theta}\right]^{1/\theta} + 2\right),
\]
where $C_{\alpha,\theta}$ depends only on $(\alpha,\theta)$ in
Lemma~\ref{lem:key1}.
\end{theorem}

\begin{remark}
For large $n$, $\DD_\alpha(\pi_n\|\pi_0)\asymp\log n$, yielding a reward bound of
order $\log^{1/\theta}n$.  This matches the scaling suggested by maximal
inequalities and shows that, unlike KL-regularized RLHF, best-of-$n$ policies do
not exhibit severe catastrophic Goodhart behavior under sub-Weibull rewards.
\end{remark}

\section{Numerical Studies}\label{sec:experiments}

In this section, we connect the theory to computation in three steps. We first examine the tightness of the proposed generalization bounds in a controlled example where single-scale information bounds fail. We then study
\Renyi\!-regularized RLHF under controlled heavy-tailed reward perturbations, using an end-to-end alignment pipeline based on Group Relative Policy Optimization (GRPO; \citealt{shao2024deepseekmath}) and state-of-the-art reward models. Finally, we move from synthetic tail amplification to an adversarial token-space reward-hacking setting, where heavy-tailed reward responses arise from reward-model vulnerabilities. 

\subsection{Illustrating Tighter Generalization Bounds}\label{sec:tightness}

We revisit the two-dimensional process example of
\citet{asadi2018chaining} to illustrate a key phenomenon in the heavy-tailed,
data-adaptive setting: a single-scale mutual-information bound can be vacuous
even when a multiscale \emph{chained} information bound remains finite and
informative. Table~\ref{tab:bounds} summarizes the comparison.

\begin{table}[h]
\centering
\caption{\small $-\E[X_W]$ and upper bounds for $\theta=0.5$.
MI: Theorem~\ref{infob}; CM: Theorem~\ref{Dudley}; 
R\'enyi CMI: Theorem~\ref{chaining}.}\label{tab:bounds}
\renewcommand{\arraystretch}{1}
\begin{tabular}{c|ccccccc}
\hline
\(\varepsilon\) & $1/20$ & $1/30$ & $1/40$ & $1/50$ & $1/100$ & $1/200$ & $1/400$ \\
\hline
MI & $\infty$ & $\infty$ & $\infty$ & $\infty$ & $\infty$ & $\infty$ & $\infty$ \\
CM & 832.01 & 832.01 & 832.01 & 832.01 & 832.01 & 832.01 & 832.01 \\
R\'enyi CMI & 71.93 & 71.46 & 71.30 & 71.22 & 71.12 & 71.10 & 71.09 \\
$-\E[X_W]$ & 0.180 & 0.120 & 0.090 & 0.072 & 0.036 & 0.018 & 0.009 \\
\hline
\end{tabular}
\end{table}

Let $S=(Z_1,Z_2)$ with i.i.d. Weibull$(\theta)$ coordinates, i.e.,
$\mathbb{P}(Z_i\ge t)=\exp(-t^\theta)$ for $t\ge 0$.
Consider the canonical Weibull process on the unit circle,
\[
X_\phi := Z_1\cos\phi + Z_2\sin\phi,\qquad \phi\in[0,2\pi),
\]
with loss $l(\phi,S):=-X_\phi$.
One can verify that $\{-X_\phi\}_{\phi\in[0,2\pi)}$ is a $(\theta,C)$-sub-Weibull process (Definition~\ref{subWp}) with respect to the arc-length metric $d(\phi_1,\phi_2)=|\phi_1-\phi_2|$.


We define a randomized data-dependent selector by taking a measurable minimizer
$\phi^\star(S)\in\arg\min_\phi l(\phi,S)$ and adding a small perturbation with $\xi\perp\!\!\!\perp S$,
\[
W=\phi^\star(S)\oplus \xi \pmod{2\pi},\qquad
\mathbb P(\xi=0)=\varepsilon,\qquad
\xi\mid(\xi\neq0)\sim\mathrm{Unif}(-\pi,\pi).
\]
We compare four ways to upper bound $\E[X_W]$:
\begin{enumerate}
\item[(i)] the single-scale mutual-information (MI) bound (Theorem~\ref{infob});
\item[(ii)] the classical chaining/Dudley entropy bound (CM, Theorem~\ref{Dudley});
\item[(iii)] the chained mutual-information bound (R\'enyi CMI, Theorem~\ref{chaining});
\item[(iv)] the exact value of $-\E[X_W]$ for this example. 
\end{enumerate}
As shown in
Table~\ref{tab:bounds}, the single-scale MI bound is infinite. The chained information bound remains finite and is substantially
tighter than the classical entropy-based chaining bound.

For any $\varepsilon>0$, the conditional law $P_{W\mid S=s}$ has an atom at
$\phi^\star(s)$, whereas the marginal law $P_W$ is non-atomic. Hence
$P_{W\mid S=s}\not\ll P_W$ for $P_S$-a.e.\ $s$, implying
$I(W;S)=I_\alpha(W;S)=\infty$ and rendering the single-scale MI bound
vacuous. In contrast, the quantized variables $[W]_k$ used in the chaining
construction take values in a \emph{finite} partition whose cells all have
strictly positive marginal probability, thanks to the uniform component of
$\xi$. Therefore each $I_\alpha([W]_k;S)$ is finite, and the chained bound
remains informative.


\subsection{RLHF under Heavy-Tailed Rewards}\label{sec:rlhf-experiments}

To illustrate, the mechanism by which R\'enyi-regularized RLHF mitigates catastrophic Goodhart effects, we empirically study R\'enyi-regularized RLHF under \emph{controlled} heavy-tailed reward perturbations. Since commonly used open-source reward models often appear approximately light-tailed on in-distribution prompts, we emphasize that the goal of these experiments is not to claim that the underlying reward model is intrinsically heavy-tailed; rather, we construct heavy-tailed training signals in a principled way to isolate the interaction between tail behavior and the choice of divergence regularizer. Starting from the constrained formulation~\eqref{optRENYI}, we adopt the
Lagrangian optimization:
\begin{equation}\label{optRENYI_penalty}
\pi_{\beta}^{*}
=
\arg\max_{\pi}\;
\mathbb{E}_{x \sim \mathcal{D}}[
\mathbb{E}_{y \sim \pi(\cdot\mid x)}[ r(x,y)]-
\beta\, D_\alpha(\pi(\cdot\mid x) \,\|\, \pi_{0}(\cdot\mid x))],
\quad \beta \ge 0,
\end{equation}
where $\pi_0(\cdot\mid x)$ is the reference policy, $\alpha\ge 1$ is the \Renyi
order (with $\alpha=1$ reducing to KL), and $\beta$ controls the trade-off between
reward maximization and conservatism.

\paragraph{Pipeline.}
We implement an end-to-end alignment pipeline based on Group Relative Policy
Optimization (GRPO; \citealt{shao2024deepseekmath}) with three stages:
\begin{itemize}
\item[(i)] \textbf{Supervised fine-tuning (SFT).}
We start from Qwen2.5-1.5B~\citep{qwen2.5} and perform supervised fine-tuning on \texttt{UltraChat-200k} to obtain an initial policy
that reliably follows the dialogue format and instructions. All runs are initialized from the same SFT checkpoint.

\item[(ii)] \textbf{Reward modeling.}
We train a lightweight proxy reward model (Qwen2.5-0.5B) on the
\texttt{hh-rlhf} training split to provide scalable training-time feedback.
The proxy model takes concatenated \texttt{prompt+completion} pairs as input
and is used \emph{only} to produce online rewards during GRPO. For
evaluation, we additionally use a stronger reward model,
Skywork-Reward-V2-Qwen3-1.7B~\citep{liu2025skywork}, as a surrogate evaluator
for the (unobserved) gold reward; its output is never used during training. For simplicity, we refer to its scores as the \emph{gold reward}. This surrogate evaluator is trained on a much larger dataset and achieves higher
correlation with human judgments, but it is too computationally expensive to
run routinely during RL optimization. Throughout this section, we refer to the lightweight
training-time model's scores as the \emph{proxy reward}.

\item[(iii)] \textbf{RL optimization and evaluation.}
We optimize the SFT policy via GRPO using the proxy reward model as the training-time critic.
We use a frozen reference model initialized from the same checkpoint as the policy to compute divergence terms and to
recalibrate reward statistics. This two-reward-model protocol is standard in the RLHF literature
(e.g., \citealt{kwa2024catastrophic}) and provides a practical balance between computational efficiency and evaluation fidelity.
\end{itemize}

\paragraph{Experimental settings.}
We fix $\beta=1$, sweep $\alpha\in\{1,2,\ldots,10\}$, train on the \texttt{hh-rlhf} training split for up to $1500$ GRPO steps, and evaluate every $10$ steps on 128 held-out prompts. Full training and evaluation hyperparameters (generation settings, batch sizes, optimizer details, truncation lengths, and reward recalibration) are provided in Appendix~\ref{app:rlhf_details}.\\

\paragraph{Inducing heavier tails as controlled stress tests.}
Most empirical RLHF analyses implicitly treat reward-model scores as approximately Gaussian, or more generally light-tailed~\citep{yan2024reward,yang2024bayesian}.
Our theory predicts that the choice of divergence regularizer matters most when rare, high-magnitude reward events are plausible, possibly due to distribution shift, reward-model miscalibration, or occasional outlier completions. Since widely used open-source reward models often appear approximately light-tailed on in-distribution prompts, we do not claim that the base proxy reward is intrinsically heavy-tailed; instead, we \emph{stress-test} this regime by constructing heavy-tailed training signals in controlled ways while keeping the rest of the pipeline fixed. Our construction follows \cite{bakhshizadeh2023sharp} and applies an order-preserving power transform to the recalibrated proxy reward before optimization:
\[
\tilde r(x,y)=\operatorname{sign}\!\big(r(x,y)\big)\,\big|r(x,y)\big|^{\kappa},
\qquad \kappa \in \{1,6,8,10\}.
\]
Here, $\kappa=1$ corresponds to the original reward. The $t\mapsto \operatorname{sign}(t)|t|^{\kappa}$ is strictly increasing on $\mathbb{R}$ for $\kappa>0$ and therefore preserves the preference ordering induced by the original proxy score. Moreover, if the original proxy score is sub-Gaussian, then the transformed reward $\tilde r$ is sub-Weibull with parameter \(2/\kappa\).


To confirm that the power transform amplifies the tail behavior of the training signal, Figure~\ref{gold_proxy_trans} compares the empirical distributions and QQ plots of the original and transformed proxy rewards for the representative choice \(\kappa=8\). The upper panels show that the original reward \(r(x,y)\) is relatively concentrated and visually close to a light-tailed profile, whereas the transformed reward \(\tilde r(x,y)\), overlaid with kernel density estimates, becomes much more dispersed with substantially more mass in the extremes. The lower panels provide the corresponding QQ plots against the normal reference distribution. The pronounced deviations in the tails after transformation further indicate heavier-than-normal tail behavior. These results support that the power transform makes rare, large-magnitude reward realizations substantially more influential during training.

\begin{figure}[t]
\centering
\includegraphics[scale=0.3]{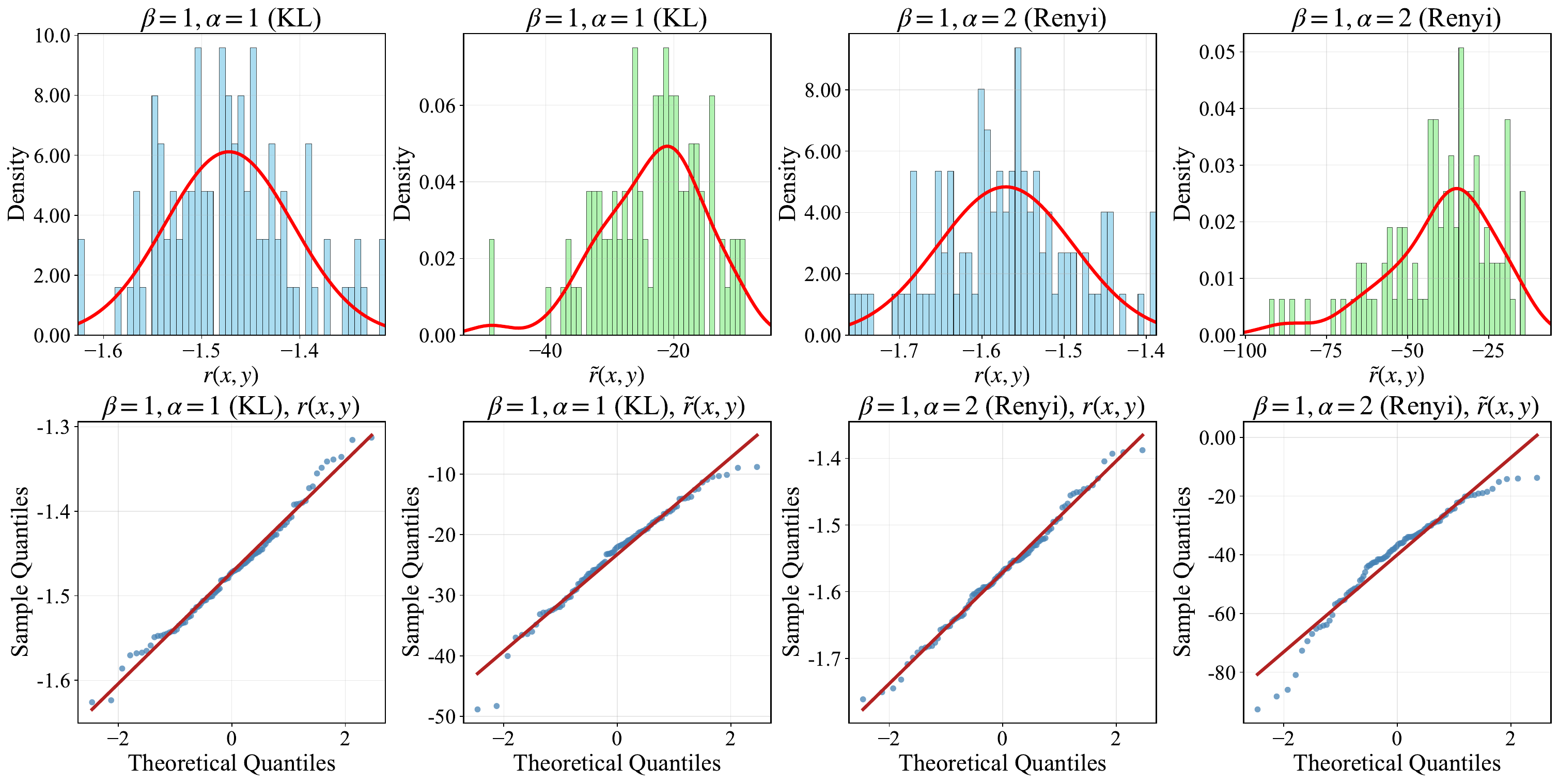}
\caption{\small Empirical distributions and QQ plots of proxy reward model scores before and after the power transformation for \(\kappa=8\). The upper panels show the empirical distributions of the original reward \(r(x,y)\) and the transformed reward \(\tilde r(x,y)=\operatorname{sign}(r(x,y))|r(x,y)|^{8}\) under \(\beta=1\) with \(\alpha=1\) (KL) and \(\alpha=2\) (R\'enyi). The original rewards are overlaid with fitted normal density curves, whereas the transformed rewards are overlaid with kernel density estimates. The lower panels report the corresponding QQ plots against the normal reference distribution. The transformation
preserves reward ordering while substantially amplifying the heavy-tailedness
of the training signal.}
\label{gold_proxy_trans}
\end{figure}

\noindent\textit{Robustness across heavy-tail constructions.}
The order-preserving power transform above provides a convenient knob (via $\kappa$) to dial tail-heaviness without changing preference rankings, but it is only one way to generate rare, high-impact reward outliers.
To verify that our conclusions are not an artifact of this particular transformation, we also study an alternative heavy-tailed construction based on adding centered Weibull noise to the reward signal prior to optimization; see Appendix~\ref{app:add_experiments}.
Across both constructions (and across sweeps of $\kappa$ and noise parameters), we observe the same qualitative pattern: KL-style control exhibits unstable reward--divergence behavior and a non-monotone proxy--gold relationship once rewards become heavy-tailed, whereas higher-order \Renyi regularization yields smoother reward--divergence curves and a more monotone proxy--gold coupling.\\

\paragraph{Reward--divergence behavior.} Figures~\ref{divergence_vs_reward11}--\ref{divergence_vs_reward1210} show the empirical relationship between reward and divergence under $\beta=1$.
Under KL regularization ($\alpha=1$), Figure~\ref{divergence_vs_reward11} shows a clear separation between the proxy reward (training-time signal) and the \emph{gold reward} (evaluation signal) once the proxy reward is heavy-tailed ($\kappa>1$). As the policy moves farther from the reference model, the proxy reward generally continues to increase, whereas the gold reward exhibits a pronounced \emph{rise-then-collapse} pattern: moderate divergence is initially beneficial, but additional divergence can eventually substantially degrade proxy-gold performance. This is a Goodhart-style signature of proxy over-optimization, in which continued improvement on the training-time reward no longer translates into improvement under the stronger evaluator.

This pattern is consistent with the mechanism in Section~\ref{se:whyKLfails}. Occasional extreme proxy rewards can dominate the policy-gradient estimate and induce abrupt, localized changes in the policy. Because the KL penalty averages discrepancies over the full action distribution, such tail-driven deviations need not incur a proportionate KL cost. The optimization can therefore drift toward policies that exploit proxy-specific artifacts while remaining relatively close to the reference policy in KL, which explains the collapse in gold reward in Figure~\ref{divergence_vs_reward11}.

Figure~\ref{divergence_vs_reward1210} shows that this behavior changes under R\'enyi regularization. For $\alpha>1$, higher-order R\'enyi divergences penalize large likelihood ratios $\pi(\cdot\mid x)/\pi_0(\cdot\mid x)$ more strongly, even when they arise on small subsets of actions. This makes concentrated, outlier-driven departures from the reference policy more expensive and thereby dampens the instability caused by heavy-tailed proxy rewards. Empirically, across $\alpha\in\{2,\ldots,10\}$, the gold reward is substantially more stable and typically increases before leveling off, rather than exhibiting the sharp collapse seen under KL. Taken together, Figures~\ref{divergence_vs_reward11} and~\ref{divergence_vs_reward1210} suggest that R\'enyi regularization provides a more reliable trust-region signal in the heavy-tailed regime.

\begin{figure}[t]
\centering 
\includegraphics[scale=0.39]{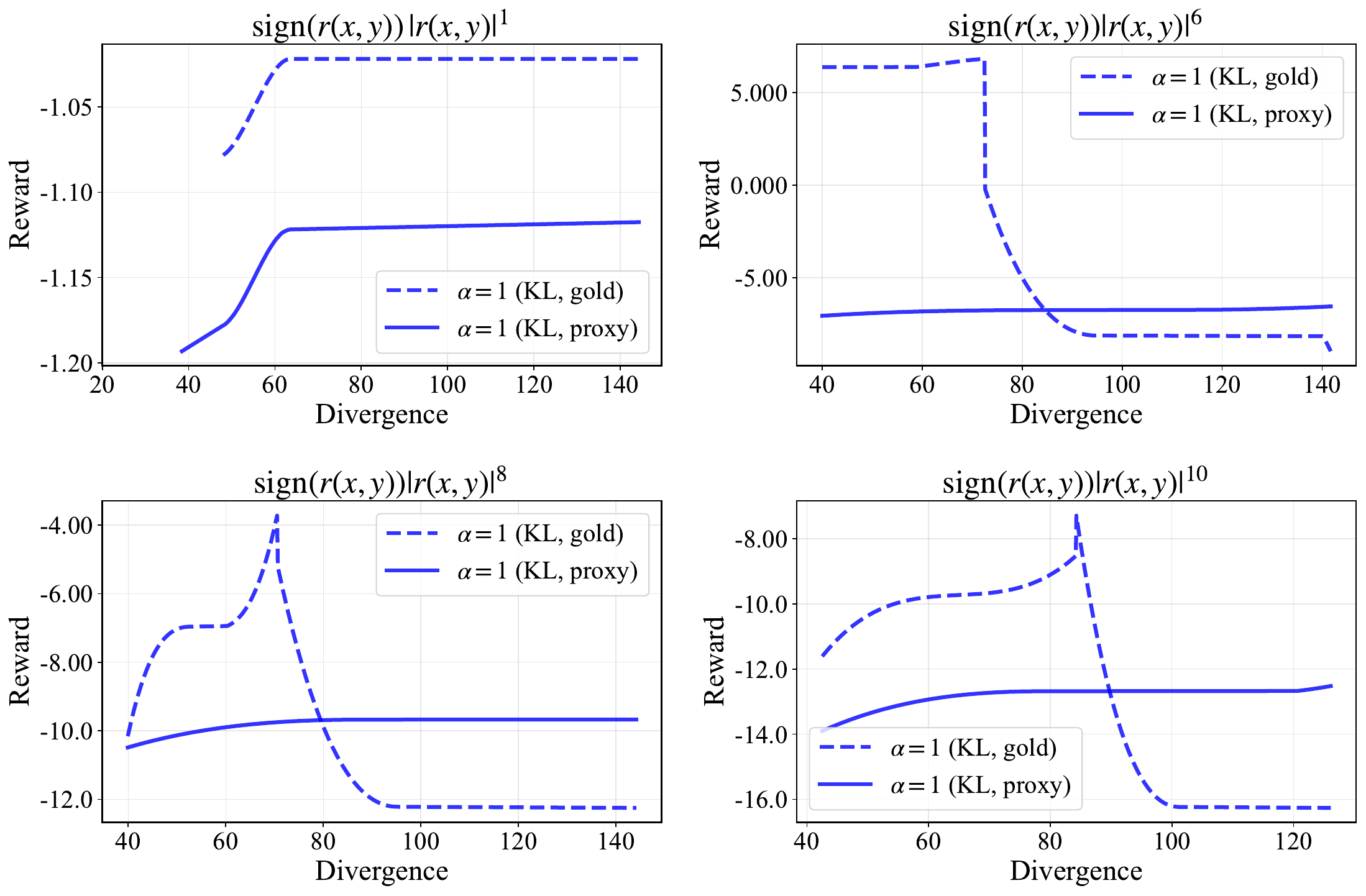}

\caption{ \small
Reward--divergence relationship for $\alpha=1$ (KL) with $\beta=1$ under different power exponents $\kappa$.
The four subplots correspond to $\kappa=1,6,8,10$, ordered from left to right and top to bottom.
As $\kappa$ increases, a clearer separation emerges between proxy and gold rewards: the proxy reward tends to keep increasing with divergence, whereas the gold reward exhibits a rise-then-collapse pattern, indicating over-optimization under heavy-tailed proxy rewards.
}
\label{divergence_vs_reward11}
\end{figure}

\begin{figure}[t]
\centering 
\includegraphics[scale=0.24]{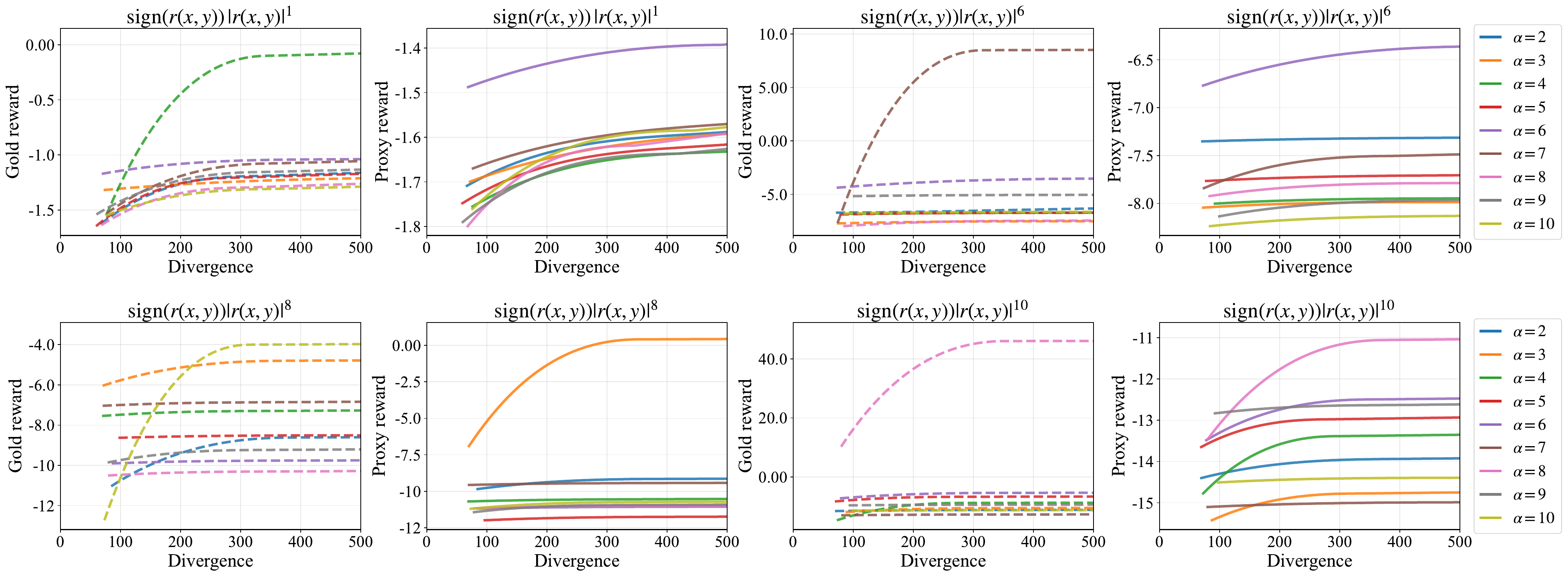}

\caption{\small
Reward--divergence relationship under R\'enyi regularization with $\beta=1$, for orders $\alpha\in\{2,\ldots,10\}$ and different power exponents $\kappa$.
The four subplots correspond to $\kappa=1,6,8,10$, ordered from left to right and top to bottom.
Across $\alpha\in\{2,\ldots,10\}$, the gold reward is substantially more stable and typically increases before leveling off, rather than exhibiting the sharp rise-then-collapse pattern observed under KL, suggesting that R\'enyi regularization provides a more reliable trust-region signal in the heavy-tailed regime.
}
\label{divergence_vs_reward1210}
\end{figure}

\paragraph{Proxy--proxy-gold coupling.}
Figure~\ref{proxy_vs_gold} examines how improvements in the proxy reward translate into improvements in the \emph{gold reward}. Under KL regularization ($\alpha=1$), this relationship becomes clearly non-monotone once the proxy reward is heavy-tailed ($\kappa>1$): the gold reward initially increases with the proxy reward, but beyond a certain point it begins to decline. Thus, continued optimization of the proxy reward can coincide with worse performance under the stronger evaluator, which manifests the over-optimization.

By contrast, under R\'enyi regularization ($\alpha\ge 2$), the proxy--gold curves remain monotone increasing and noticeably smoother across the observed range. It indicates that gains in the proxy reward are more consistently aligned with gains in the gold reward.
Overall, these results suggest that R\'enyi regularization yields a more reliable coupling between proxy and proxy-gold objectives, and therefore a more reliable trust-region signal than KL in the heavy-tailed regime.\\

\begin{figure}[t!]
\centering 
\includegraphics[scale=0.35]{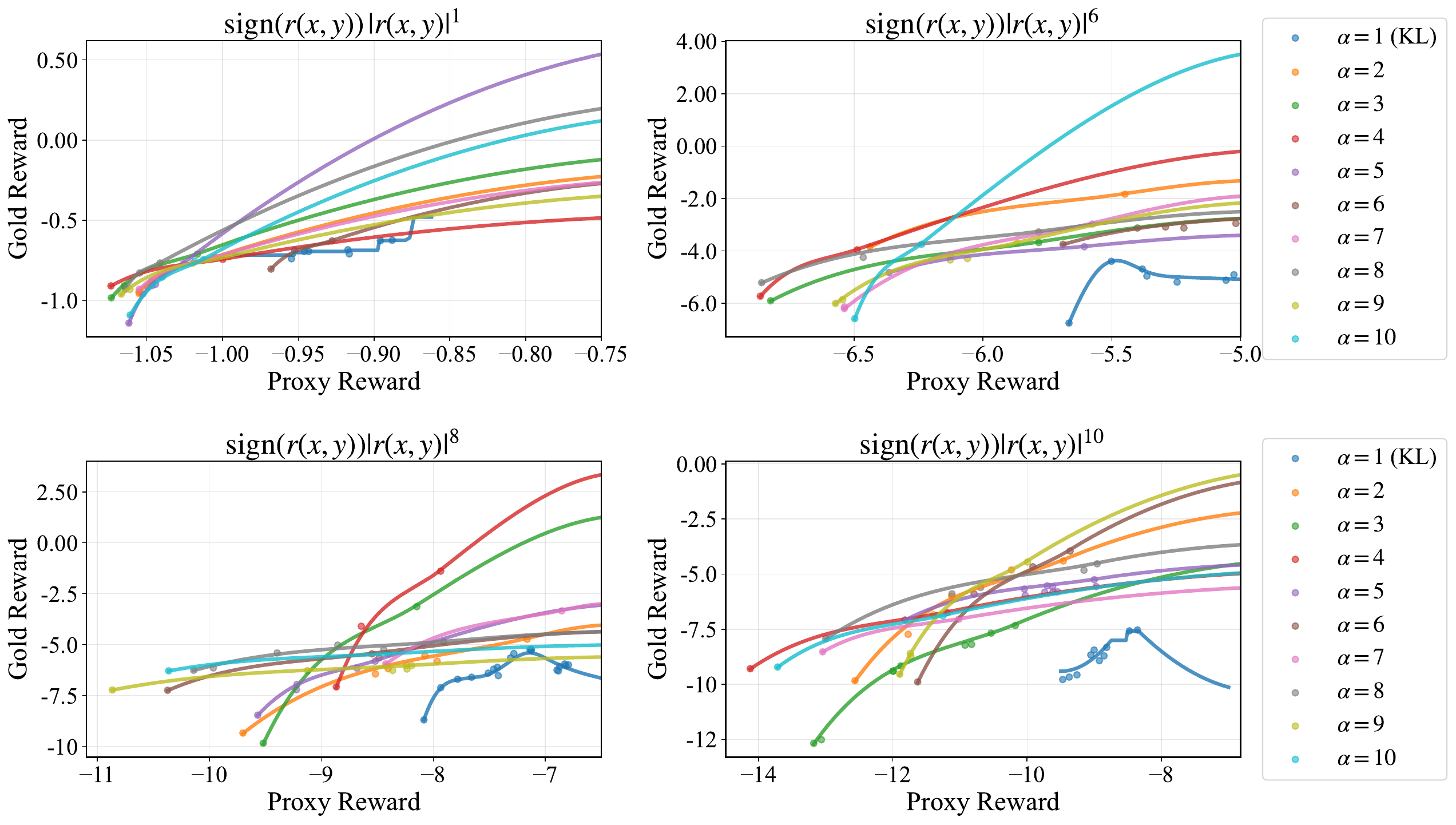}

\caption{\small
Proxy--proxy-gold reward relationship under $\beta=1$, for R\'enyi orders $\alpha\in\{1,\ldots,10\}$ and power exponents $\kappa\in\{1,6,8,10\}$.
The four subplots correspond to $\kappa=1,6,8,10$, ordered from left to right and top to bottom.
As $\kappa$ increases, the KL case ($\alpha=1$) becomes clearly non-monotone, with the gold reward eventually declining as the proxy reward continues to improve, whereas the curves for $\alpha\ge 2$ remain smoother and largely monotone increasing, indicating a more reliable coupling between proxy and gold objectives.
}
\label{proxy_vs_gold}
\end{figure}


\paragraph{Additional robustness checks.}
Appendix~\ref{app:add_experiments} provides additional diagnostics and robustness checks, including tail plots for the power transformation and the full set of results for the additive Weibull-noise construction. Having established the same qualitative pattern under controlled heavy-tailed perturbations, we next examine a more adversarial source of heavy tails: token-space attacks that directly exploit reward-model vulnerabilities.

\subsection{RLHF under Token-Space Reward Attacks}\label{sec:tompa-experiments}\label{app:tompa}\label{app:toma}

The controlled experiments above isolate the effect of tail behavior by modifying
the reward signal while keeping the optimization pipeline fixed. We now consider
a less stylized experiment in which heavy-tailed rewards emerge from
reward-model exploitation itself. \citet{zhang2026beyond} introduced TOMPA and
identified a striking failure mode of reward modeling in RLHF: reward-model
vulnerabilities are not confined to semantic manipulation. Since an attacker
can search directly in token space for non-semantic adversarial patterns that
induce abnormally large reward values. The resulting high-reward outputs may
degenerate into nonsensical text while still receiving favorable evaluations
from the reward model. This phenomenon is closely aligned with the theory above:
rare token sequences can trigger disproportionately large rewards, and those
extreme values can dominate policy updates even when most rollouts receive
moderate rewards. Thus, we first inspect the tail behavior of the reward
distribution and then ask whether gains in the optimized target reward remain
aligned with an external reward diagnostic.

Building on the RLHF setup in Section~\ref{sec:rlhf-experiments}, we compare KL
regularization with R\'enyi regularization at $\alpha\in\{2,4\}$, using
$\beta=1$. Following TOMPA, we use Skywork-Reward-V2-Llama-3.1-8B
\citep{liu2025skywork} as the frozen target reward model and Qwen3-1.7B
\citep{qwen2025technical} as the attack policy. We train on 5{,}000 prompts
sampled from WildChat \citep{zhao2024wildchat,deng2024wildvis} using GRPO with
eight rollouts per prompt. For reward computation, only the generated response
token IDs are passed directly into the target reward model; the prompt is
excluded, and token IDs exceeding the target vocabulary are clipped to its
largest valid index. We update only rank-32 LoRA parameters of the policy, while
the policy backbone and reward model remain frozen. The maximum prompt and
response lengths are 2{,}048 and 512 tokens, respectively, and training uses a
learning rate of $10^{-6}$. As an online diagnostic, we periodically decode the
mapped response IDs with the target-model tokenizer and score the resulting
answer-only text using the frozen OpenAssistant reward model
\texttt{reward-model-deberta-v3-large-v2} \citep{openassistant2023rewarddeberta}.
We treat this auxiliary diagnostic score as the \emph{gold reward} in this
subsection. It is recorded only for monitoring and does not affect policy
optimization.

\begin{figure}[t]
\centering
\includegraphics[scale=0.25]{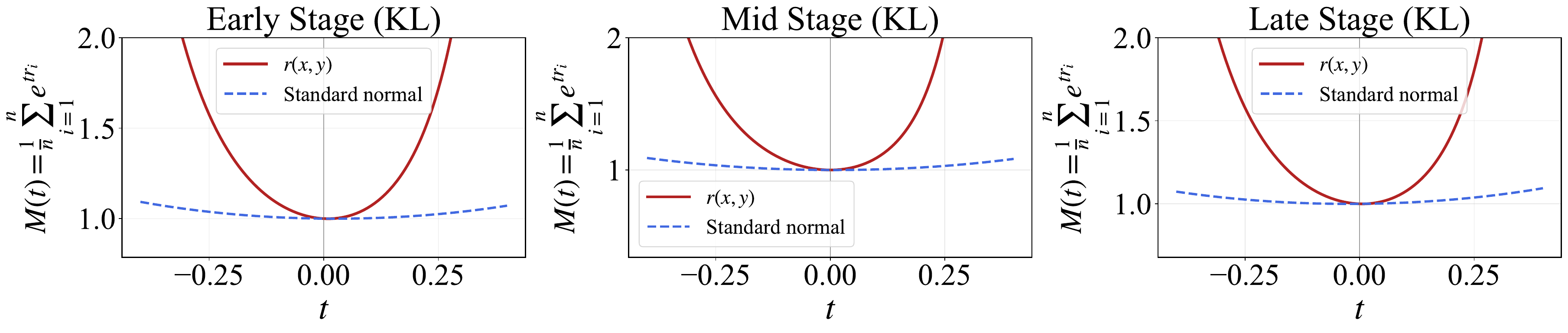}

\includegraphics[scale=0.25]{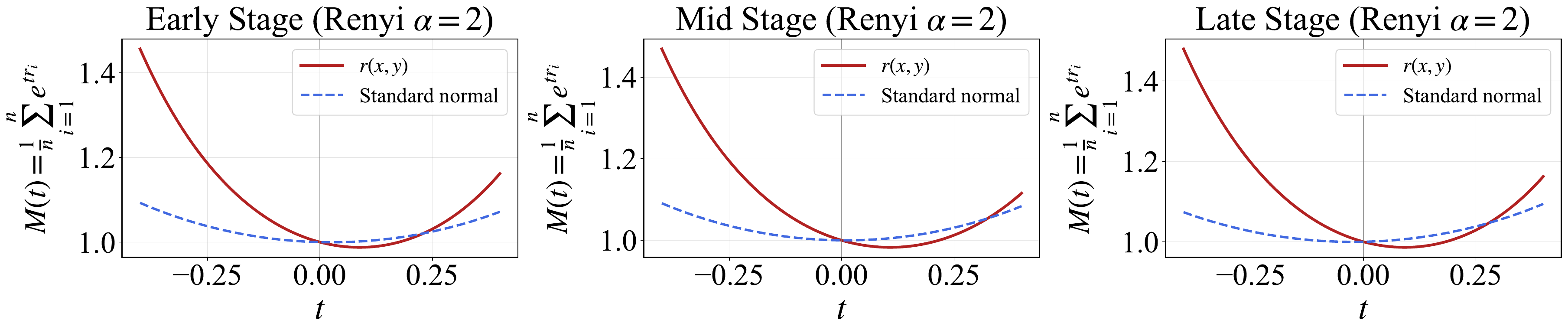}

\includegraphics[scale=0.25]{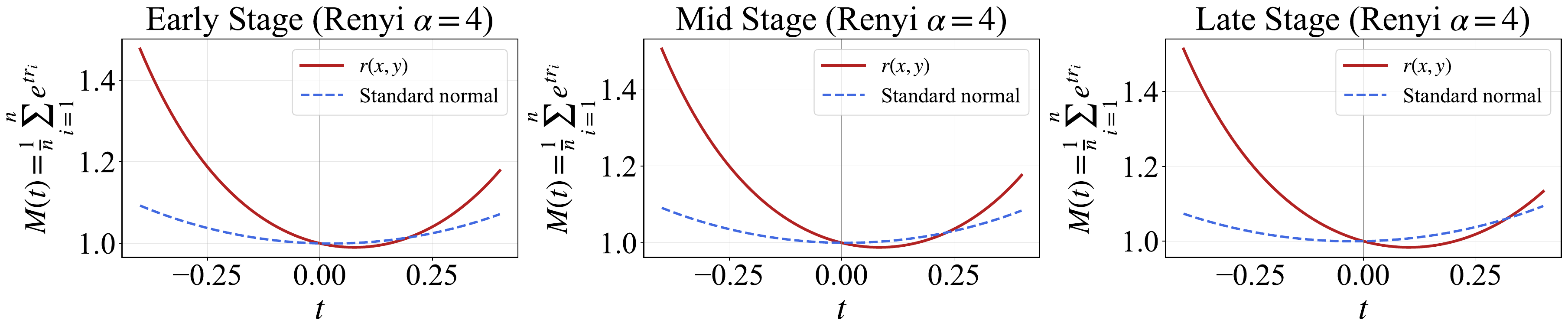}
\caption{\small Tail behavior of reward distributions under TOMPA token-space attacks with KL and R\'enyi regularization. The top row shows KL regularization, while the middle and bottom rows correspond to R\'enyi regularization with $\alpha=2$ and $\alpha=4$, respectively. Left panels plot the empirical MGF $M(t)$ against the standard normal benchmark. KL regularization exhibits stronger right-tail inflation and upper-tail deviations, while R\'enyi regularization attenuates these tail effects, especially when $\alpha=4$.}
\label{fig:tompa_tail_diagnostics}
\end{figure}

Figure~\ref{fig:tompa_tail_diagnostics} shows that KL regularization produces
more pronounced heavy-tailed reward behavior under token-space attack: the
empirical MGF grows faster than the normal benchmark. In contrast, R\'enyi regularization reduces the
upper-tail deviation, with a stronger attenuation at $\alpha=4$. This suggests
that R\'enyi regularization can suppress rare but abnormally high reward values
induced by token-space exploitation
\begin{figure}[t]
\centering
\includegraphics[width=0.98\textwidth]{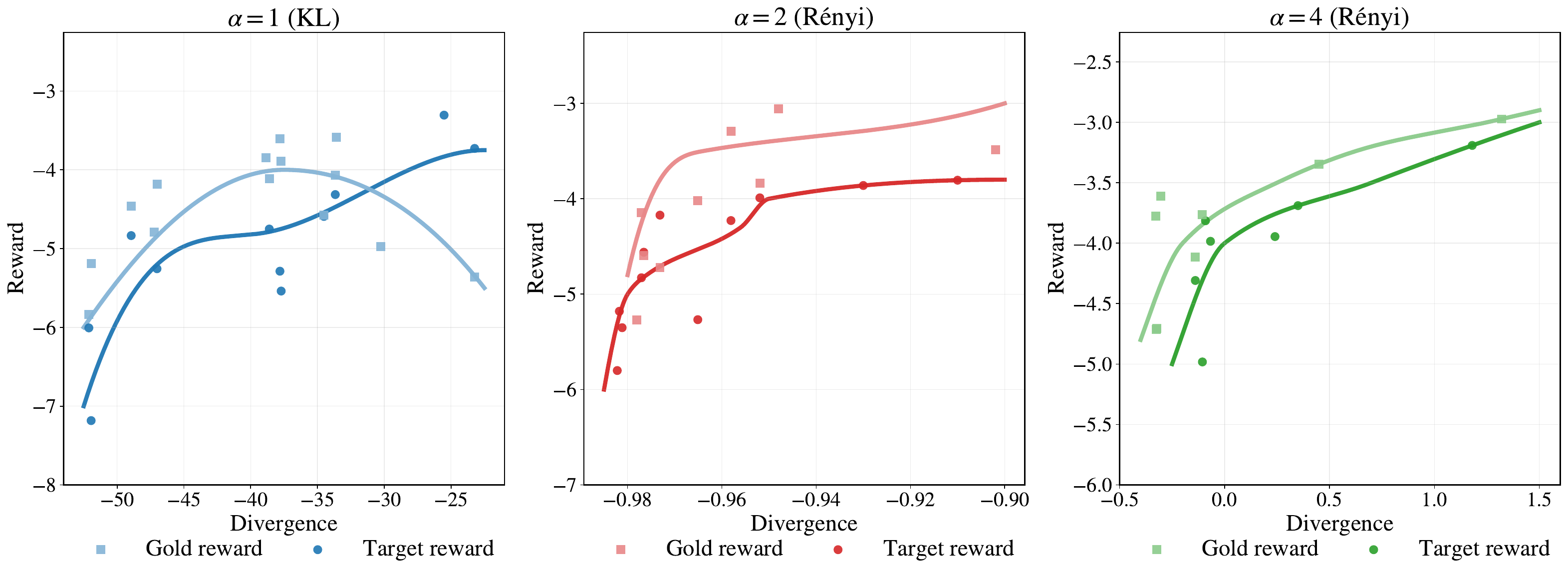}
\vspace{0.8em}
\includegraphics[width=0.98\textwidth]{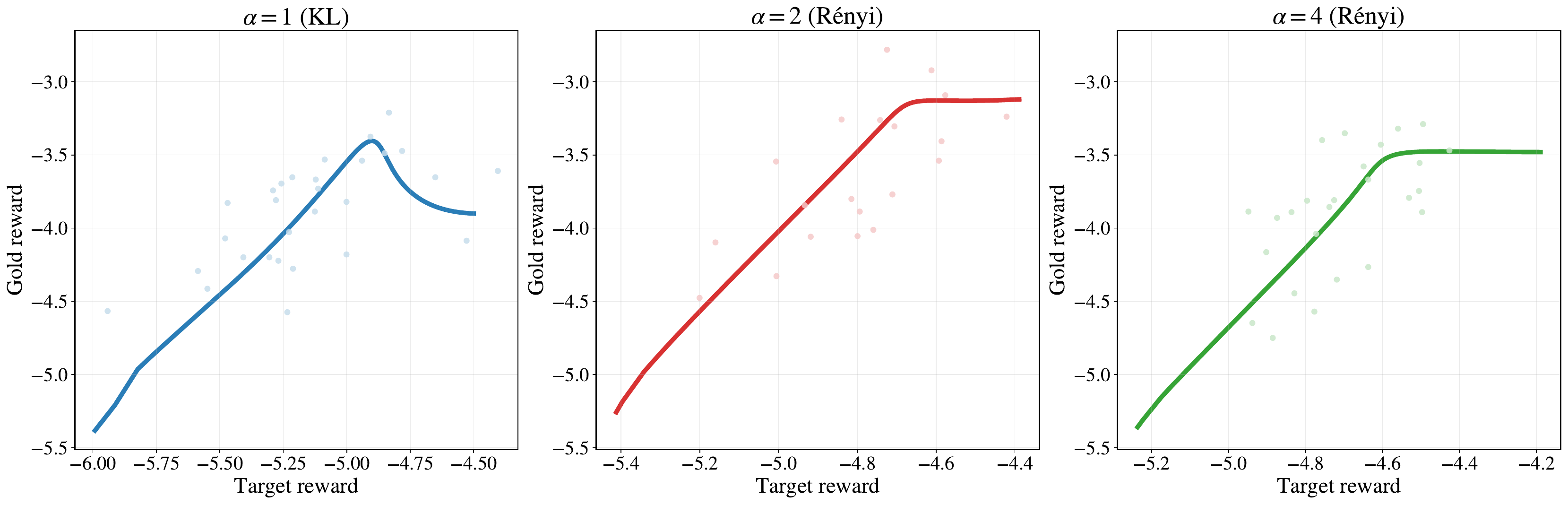}
\caption{\small Reward--divergence diagnostics under TOMPA token-space attacks with KL and R\'enyi regularization. The upper panel reports the relationship between reward and divergence for $\alpha=1$ (KL), $\alpha=2,4$ (R\'enyi), where points denote observed rewards and curves denote smoothed trends. The lower panel compares the directly optimized target reward with the auxiliary gold reward, obtained by decoding the mapped response IDs and scoring the resulting answer-only text with an external frozen reward model. Under KL regularization, the target reward increases as divergence grows, whereas the gold reward first increases and then decreases. Under R\'enyi regularization, both rewards tend to increase together and gradually stabilize.}
\label{fig:reward_divergence_diagnostics}
\end{figure}
Figure~\ref{fig:reward_divergence_diagnostics} further explains how this tail
behavior translates into reward hacking. Under KL regularization, the target
reward continues to rise as divergence increases, but the gold reward first
increases and then declines. The lower panel gives a direct view of this
decoupling: in the high-target-reward region, larger optimized rewards no longer
correspond to better decoded outputs under the external diagnostic. The policy
therefore appears to exploit token-space vulnerabilities of the target reward
model rather than produce responses with improved reward behavior.

By contrast, R\'enyi regularization yields a more stable relationship between
the optimized target reward and the gold reward. For $\alpha=2$ and $\alpha=4$,
both rewards tend to increase together as divergence grows, and the gold reward
approaches a plateau rather than declining in the high-target-reward region.
Combining Figures~\ref{fig:tompa_tail_diagnostics}
and~\ref{fig:reward_divergence_diagnostics}, we find that KL regularization
leaves two related vulnerabilities in this adversarial setting: heavier-tailed
reward responses and a mismatch between the optimized target reward and the
external gold-reward diagnostic. R\'enyi regularization mitigates both effects,
supporting the broader conclusion that tail-aware divergence control is useful
not only for controlled heavy-tailed perturbations, but also for reward-model
exploitation that creates heavy tails endogenously.

\section{Conclusions and Discussions}\label{sec:conclusion}

This paper develops a tail-aware information--theoretic framework for
data-dependent selection and adaptive suprema under sub-Weibull tails. The
starting point is the observation that KL mutual-information bounds,
while powerful in light-tailed settings, can be ineffective when MGFs do not
exist and rare extreme events dominate selection effects. To address this, we
introduce a decorrelation lemma tailored to sub-Weibull tails: it controls
expectations under a changed measure through a shifted-log $f_\theta$-divergence
and relates this quantity to \Renyi divergence. This route avoids MGF arguments
and works directly with sub-Weibull Orlicz norms.

On the empirical-process side, we extend maximal and chaining tools to genuinely
heavy-tailed regimes. We establish a sharp maximal inequality for heavy-tailed
sub-Weibull random variables and a Dudley-type entropy bound for sub-Weibull
processes. The resulting rates make the role of the tail index explicit: maxima
scale as $(\log n)^{1/\theta}$, while metric entropy enters the Dudley integral
with exponent $1/\theta$.

Combining the decorrelation and empirical-process ingredients, we obtain
tail-adaptive information bounds for data-dependent selectors. The single-scale
bound replaces KL mutual information with shifted-log $f_\theta$-mutual
information and its \Renyi refinements. The multiscale result then lifts this
control to a Dudley-type information inequality through the discretized variables
$[W]_k$. This chaining formulation is especially useful when a single-scale
information term is infinite or vacuous, but the information revealed at each
resolution remains controlled.

We apply the same tail-aware viewpoint to RLHF under heavy-tailed and
adversarial rewards. The theory shows why KL-constrained RLHF can suffer
catastrophic Goodhart under heavy-tailed proxy rewards, and why
\Renyi-regularized RLHF controls reward inflation at a fixed divergence budget.
We also analyze best-of-$n$ policies and show that their reward growth is
controlled through \Renyi divergence and the sub-Weibull maximal rate. The
numerical studies support this picture in two settings: controlled heavy-tail
perturbations of reward signals and TOMPA-style token-space reward attacks. In
both cases, higher-order \Renyi regularization attenuates tail inflation and
improves the coupling between proxy and gold reward diagnostics relative to KL.

Several limitations and directions remain. First, the sharpness of the
\Renyi-based constants and exponents should be better understood, especially in
structured or high-dimensional settings where localization may improve the
bounds. 
Second, while this work focuses on shifted-log
$f$-divergences and \Renyi divergences with $\alpha>1$, other information
measures, such as $\alpha<1$ \Renyi divergences, Wasserstein distances, or hybrid
integral probability metrics, may provide complementary forms of tail control. Third, adaptively estimating \Renyi order $\alpha$ is an important next step. Similarly, \cite{ferrari2010maximum} choose $\alpha$ by minimizing an estimated asymptotic mean squared error of the maximum $L\alpha$-likelihood estimator.

From a practical standpoint, the results suggest moving beyond KL penalties when
rare high-reward events can dominate optimization. For RLHF, \Renyi-type
regularization offers a way to penalize concentrated tail exploitation without
simply forbidding policy movement. The token-space attack experiment further
suggests that such control is relevant not only for synthetic heavy-tailed
perturbations, but also when reward-model vulnerabilities create heavy tails
endogenously. We hope these tools contribute to a broader tail-adaptive
information--theoretic theory for modern learning and alignment systems. 

\section*{Acknowledgements}
\noindent H.Z. is supported in part by Beijing Advanced Innovation Center for Future Blockchain and Privacy Computing and the Fundamental Research Funds for the Central Universities (No. KG16447001). W.T. is supported in part by the Funds of the Natural science Foundation of Hangzhou (No.~2025SZRJJ1388) and the Beijing Outstanding Young Scientist Program (No.~JWZQ20240101027).  Q.S. is supported in part by MBZUAI, the Natural Sciences and Engineering Research Council of Canada (Grant RGPIN-2018-06484), computing resources provided by the Digital Research Alliance of Canada.

\bibliography{ref}

\newpage
\appendix

\section*{Appendix}

\startcontents[sections] 
\printcontents[sections]{}{1}{} 


\section{Preliminaries}\label{app:prior}

\subsection{Preliminaries on sub-Weibull random variables}\label{app:prelim}


We review the definition of Orlicz norm and an equivalent definition for sub-Weibull random variables. 

\begin{definition}[Orlicz norm, \cite{Aad2023}]
For a random variable $X$, the Orlicz norm associated with a Young function $\psi$ is defined as
\[
\|X\|_{\psi} := \inf \left\{ K > 0 : \mathbb{E}\,\psi\!\left(\frac{|X|}{K}\right) \le 1 \right\},
\]
where $\psi : [0,\infty) \to [0,\infty)$ is a Young function, i.e., a convex function satisfying $\psi(0)=0$ and $\lim_{x \to \infty} \psi(x) = \infty$. 
\end{definition}

For example, the sub-Weibull Orlicz norm is defined as
\[
\|X\|_{\psi_\theta}:=\inf\left\{K>0:\ \mathbb{E}\exp\!\left(\left|\frac{X}{K}\right|^\theta\right)\le 2\right\}.
\]
The $\|\cdot\|_{\psi_\theta}$ is a norm when $\theta\ge 1$. For $0<\theta<1$, it is only a quasi-norm: the triangle inequality fails, but a weak triangle inequality holds~\cite[Lemma A.3]{gotze2021concentration}.

As mentioned in the introduction, sub-Weibull random variables can also be equivalently characterized via their tail probabilities. Below we give a definition equivalent to the one based on the Orlicz norm.

\begin{definition}[Sub-Weibull random variables]\label{def:subweibull}
A random variable $X$ is called \emph{sub-Weibull} with tail parameter $\theta>0$, denoted $X \sim \subw(\theta)$, if there exist positive constants $a,\, b>0$ such that
\[
\mathbb{P}(|X| \ge x) \le a \exp \bigl\{-(x/b)^{\theta}\bigr\}, \qquad x>0.
\]
\end{definition}

In other words, sub-Weibull distributions have tails no heavier than those of a Weibull random variable with the same parameter $\theta$. When $\theta < 1$, $X$ is heavy-tailed, whereas $\theta \ge 1$ corresponds to light-tailed distributions, including sub-Gaussian ($\theta=2$) and sub-exponential ($\theta=1$) cases.  
The next section reviews key prior results on maximal inequalities for light-tailed sub-Weibull random variables (sub-Weibull with $\theta\geq 1$, \cite{Aad2023}), and information-theoretic generalization bounds for sub-Gaussian random variables \citep{Russo20,chu2023unified}.

\subsection{Maximum inequalities for light-tailed sub-Weibull random variables}\label{se:ap-max}

The following maximal inequality for light-tailed sub-Weibull random variables with $\theta\geq 1$ is due to \cite{Aad2023}. 

\begin{proposition}[Maximal inequality for light-tailed sub-Weibull random variables]\label{prp-MaximalSW}
For $\theta \ge 1$, let $\{X_{i}\}_{i=1}^n$ be random variables with $\max_{1 \le i \le n}\|X_i\|_{\psi_\theta} < \infty$ and $\psi_\theta(x)=e^{x^\theta}-1$,
\begin{equation}\label{eq:MaximalSW1}
\E \Bigl( \max_{1 \le i \le n} |X_i| \Bigr)
\;\le\; \psi^{-1}_{\theta}(n) \max_{1 \le i \le n}\|X_i\|_{\psi_{\theta}}
= (\log (1+n))^{1 / \theta}\, \max_{1 \le i \le n}\|X_i\|_{\psi_{\theta}}.
\end{equation}
\end{proposition}

The proof of Proposition~\ref{prp-MaximalSW} relies on Jensen's inequality, which requires the convexity of $\psi_\theta$. This property holds only when $\theta \ge 1$ and fails for $\theta<1$. Note that the sub-Gaussian case corresponds to $\theta = 2$. 

More generally, \citet{Russo20} extended maximal inequalities in the sub-Gaussian setting by relating them to mutual information. Under the assumption that $\{X_{i}\}_{i=1}^n$ are i.i.d. $\sigma^2$-sub-Gaussian, they established:
\begin{equation}\label{eq:RZ}
    \E \bigl[ X_W \bigr] \leq \sqrt{2\sigma^2\, I(W;X_1, \ldots, X_n)},
\end{equation}
where the algorithm output is a random index $W = W(X_1, \ldots, X_n)$. When $W$ is taken to be the maximum selector in $[n]$, the mutual information bound yields the classical sub-Gaussian maximal inequality
\begin{equation}\label{eq:RZ1}
\mathbb{E}\bigl[\max_{ i \in [n]} X_i\bigr] \leq \sqrt{2 \log n},
\end{equation}
by using the fact that $I(W;X_1, \ldots, X_n)\le \log n$. Hence, \eqref{eq:RZ} implies \eqref{eq:RZ1} as a special case, so \eqref{eq:RZ} is more general than \eqref{eq:RZ1}.

When the index set is infinite or even uncountable, it is natural to generalize maximal inequalities to upper bounds on the supremum of empirical or stochastic processes. Dudley's inequality (Theorem 5.24 in \cite{vanhandle2016probability}) states that
\[
\mathbb{E}\left[\sup _{t \in T} X_t\right] 
\;\leq\; 6 \sum_{k \in \mathbb{Z}} 2^{-k} \sqrt{\log N\left(T, d, 2^{-k}\right)}
\]
if $\left\{X_t\right\}_{t \in T}$ is a separable 1-sub-Gaussian process of zero mean on the bounded metric space $(T, d)$, i.e.,
\[
\E \exp\bigl(\lambda(X_t-X_s)\bigr) \leq \exp\bigl(\lambda^2 d(t, s)^2 / 2\bigr)
\quad \text{for all}~t, s \in T,\; \lambda \geq 0.
\]

\subsection{Information-theoretic generalization bounds}

Building upon \citet{Russo20}, \citet{asadi2018chaining} introduced a chaining method to derive information-theoretic generalization bounds, which can be viewed as a refinement of Dudley's inequality through mutual information.

\begin{proposition}[\citet{asadi2018chaining}]\label{Chaining MI Random Process}
Assume that $\{X_t\}_{t\in T}$ is a separable sub-Gaussian process of zero mean on a bounded metric space $(T,d)$, and let $k_1(T)$ be an integer such that $2^{-\left(k_1(T)-1\right)} \geq \operatorname{diam}(T)$. Let $\{\mathcal{P}_k\}_{k=k_1(T)}^{\infty}$ be an increasing sequence of partitions of $T$, where for each $k\geq k_1(T)$, $\mathcal{P}_k$ is a $2^{-k}$-partition\footnote{Recall the definition of partition in Definition \ref{def:part}.} of $(T,d)$. Then:
\begin{enumerate}
\item
$\E X_W \leq 3\sqrt{2}\sum_{k=k_1(T)}^{\infty}2^{-k}\sqrt{I([W]_k;X_T)}.$
\item For any arbitrary $t_0\in T$,
$
\E\bigl\lvert X_W-X_{t_0}\bigr\rvert \leq 3\sqrt{2}\sum_{k=k_1(T)}^{\infty}2^{-k}\sqrt{I([W]_k;X_T)+\log 2}.
$
\end{enumerate}
\end{proposition}

When $W$ is taken to be the maximum selector, Proposition~\ref{Chaining MI Random Process} yields an infinite-series bound that refines Dudley's inequality, since $\mathbb{E}[X_W] \leq \mathbb{E}\bigl[\sup _{t \in T} X_t\bigr]$.

However, standard MGF-based techniques for deriving maximal inequalities and mutual information bounds are ineffective for heavy-tailed sub-Weibull distributions. This motivates the development of new tools to obtain sharp information-theoretic generalization bounds in the heavy-tailed regime.


\section{Proofs for Section \ref{sec:ep}}\label{se: proof2}

\subsection{Proof of Lemma \ref{pp-MaximalSW}}
\begin{proof}
The function $\psi_\theta$ is increasing on $[0,\infty)$ and convex on $[x_\theta,\infty)$.
Indeed,
\[
\psi_\theta'(x)=\theta x^{\theta-1}e^{x^\theta},
\qquad
\psi_\theta''(x)=\theta x^{\theta-2}e^{x^\theta}\bigl(\theta x^\theta-(1-\theta)\bigr),
\]
so $\psi_\theta''(x)\ge 0$ for all $x\ge x_\theta=:\paren{\frac{1-\theta}{\theta}}^{1/\theta}$. 
Let $0<\theta<1$ and define $
g_\theta(x):=\paren{\psi_\theta(x)-\psi_\theta(x_\theta)}_+$.
Then $g_\theta$ is convex and nondecreasing on $[0,\infty)$. 

Let $K:=\max_{1\le i\le n}\norm{X_i}_{\psi_\theta}$ and $M:=\max_{1\le i\le n}\frac{\abs{X_i}}{K}$.
Then, the elementary bound
$\max_{1\le i\le n} a_i \le \sum_{i=1}^n a_i~(a_i\ge 0)$ gives
\[
\E \psi_\theta(M)\le \sum_{i=1}^n \E \psi_\theta\!\paren{\frac{\abs{X_i}}{K}}\le n.
\]
By Jensen's inequality,
\[
\paren{\psi_\theta(\E M)-\psi_\theta(x_\theta)}_+= g_\theta(\E M)\le \E g_\theta(M)\le \E\psi_\theta(M)\le n.
\]
Hence,
$\psi_\theta(\E M)\le \psi_\theta(x_\theta)+n$. Since $\psi_\theta$ is increasing,
\[
\E M=\E \paren{\max_{1\le i\le n}\frac{\abs{X_i}}{K}}\le \psi_\theta^{-1}\!\paren{\psi_\theta(x_\theta)+n},
\]
and multiplying by $K:=\max_{1\le i\le n}\norm{X_i}_{\psi_\theta}$ yields \eqref{eq:MaximalSW}. The big-$O$ rate is from $\psi_\theta^{-1}$.
\end{proof}

\subsection{Optimality and comparison with classical bounds}\label{app:sec2}

The following result shows that the bound in Lemma \ref{pp-MaximalSW} is optimal in order: the rate $\log^{1/\theta} n$ cannot be improved in general.   

\begin{proposition}\label{prop:lowerbound}
Let $\{X_i\}_{i=1}^n$ be i.i.d. Weibull variables $\mathbb{P}(X\ge x)=\exp(-(x/b)^\theta)$. Then, 
\[
\E ( \max_{i\in [n]} X_i ) \gtrsim \log^{1 / \theta}(n).
\]
\end{proposition}
\begin{proof}
The upper bound follows from the proof of Lemma \ref{pp-MaximalSW}. For the lower bound:
\[
\begin{aligned}
\E ( \max_{i\in [n]} X_i ) & =\int_0^{\infty} \mathbb{P}(\max_{i\in [n]} X_i>x) \, d x= \int_0^\infty \bigl[1-(1-\exp(-(x/b)^\theta))^n\bigr] d x\\
&\geq \int_0^{b\log(n)^{1 / \theta}} \left\{1-(1-\exp(-(x/b)^\theta))^n\right\} d x\\
&\geq \int_0^{b\log(n)^{1 / \theta}} \left\{1-(1-\frac{1}{n})^n\right\}d x\geq b\log^{1 / \theta}(n) \left(1-\frac{1}{e}\right),
\end{aligned}
\]
where the third line uses monotonicity of function $t\mapsto t^n$ and the last line uses $(1-\frac{1}{n})^n\leq \frac{1}{e}$.
\end{proof}

\subsection{Comparison with norm-to-moment bounds}\label{app:max_comparison}

A standard alternative to bounding $\mathbb{E}\max_{i\in [n]} |X_i|$ is to
first control the Orlicz norm of the maximum and then convert that norm bound
into a moment bound. For instance, \citet[(S2.2)]{mendes2023concentration}
showed that
\[
\left\|\max_{i\in [n]}|X_i|\right\|_{\psi_\theta}
\le c_1 \psi_\theta^{-1}(2n)\,\max_{i\in [n]}\|X_i\|_{\psi_\theta},
\qquad\text{where}~~c_1=\frac{2}{\log(1.5)}.
\]
using \cite[Lemma~2.2.2]{Aad2023}. Combining this with the first-moment bound
for the sub-Weibull norm \cite[Corollary~3]{zhang2022sharper},
$
\mathbb{E}|X|\le 2\|X\|_{\psi_\theta}\Gamma\!\left(\frac{1}{\theta}+1\right),
$
yields the bound
\[
\mathbb{E}\Bigl[\max_{i\in [n]} |X_i|\Bigr]
\le 2c_1\,\psi_\theta^{-1}(2n)\,\Gamma\!\left(\frac{1}{\theta}+1\right)
\max_{i\in [n]}\|X_i\|_{\psi_\theta}.
\]
In contrast, Lemma~\ref{pp-MaximalSW} controls the expectation directly and
avoids the additional multiplicative factor
$\Gamma(1/\theta+1)$ introduced by the norm-to-moment conversion. The ratio
between the two prefactors is
\begin{equation}\label{eq:sharplow}
    R_{n,\theta}=\frac{4/ \log (1.5) \psi_\theta^{-1}(2 n)\Gamma\!\left(\frac{1}{\theta}+1\right)}{\psi_\theta^{-1}\bigl(\psi_\theta(x_\theta)+n\bigr)
}\ge \frac{4}{\log(1.5)} \Gamma\left(\frac{1}{\theta} + 1\right) \approx 9.865 \, \Gamma\left(\frac{1}{\theta} + 1\right).
\end{equation}
Since $\Gamma(1/\theta+1)$ grows rapidly as $\theta\downarrow 0$, the
norm-to-moment route can be substantially looser in the heavy-tailed regime,
whereas Lemma~\ref{pp-MaximalSW} retains the sharp logarithmic scaling.

To verify \eqref{eq:sharplow}, note that
$\psi_\theta(x)=e^{x^\theta}-1$ has inverse
$\psi_\theta^{-1}(y)=\bigl[\log(y+1)\bigr]^{1/\theta}$. Thus,
\begin{equation*}
    R_{n,\theta}= \frac{\frac{4}{\log(1.5)} \psi_\theta^{-1}(2n) \Gamma\left(\frac{1}{\theta}+1\right)}{\psi_\theta^{-1}\bigl(\psi_\theta(x_\theta)+n\bigr)}.
\end{equation*}
We substitute $\psi_\theta(x_\theta) = e^{x_\theta^\theta} - 1=e^{\frac{1-\theta}{\theta}} - 1$, where $x_\theta^\theta = \frac{1-\theta}{\theta}$.
Substituting the inverse functions gives
\begin{equation*}
    R_{n,\theta} = \frac{4}{\log(1.5)} \Gamma\left(\frac{1}{\theta}+1\right) \left[ \frac{\log(2n+1)}{\log(n + e^{\frac{1-\theta}{\theta}})} \right]^{1/\theta}.
\end{equation*}
As $n\to\infty$,
\begin{equation*}
    \lim_{n \to \infty} \frac{\log(2n+1)}{\log(n + e^{\frac{1-\theta}{\theta}})} = \lim_{n \to \infty} \frac{\log n + \log(2 + 1/n)}{\log n + \log(1 + e^{\frac{1-\theta}{\theta}}/n)} = 1.
\end{equation*}
Therefore, the ratio converges to the constant
$\frac{4}{\log(1.5)} \Gamma\left(\frac{1}{\theta}+1\right)$. Moreover, for any
$n \ge e^{\frac{1-\theta}{\theta}}$, we have
$\frac{\log(2n+1)}{\log(n + e^{\frac{1-\theta}{\theta}})} \geq 1$. Hence,
\begin{equation*}
    R_{n,\theta} \ge \frac{4}{\log(1.5)} \Gamma\left(\frac{1}{\theta}+1\right).
\end{equation*}

The next example shows that taking $W=\arg\max_{i\in[n]}|X_i|$ in Theorem \ref{infob} recovers the maximal growth rate $\mathbb{E}[\max_i|X_i|]\lesssim(\log n)^{1/\theta}$.

\begin{example}\label{Max}
Let $S= \{X_1, \ldots, X_n\}$ be a vector of $n$ i.i.d. continuous random
variables with a common density. Let
$
W = \argmax_{i \in [n]} |X_i|$. Then
\[
I(W;S)=\lim_{\alpha \to 1} I_\alpha (W;S) =\log n,
\]
and for every $\alpha \in (0, \infty) \setminus \{1\}$, the R\'enyi-$\alpha$
mutual information is
\[
    I_\alpha(W; S) = \log n.
\]
\end{example}
\begin{proof}
By continuity, ties in $|X_1|,\ldots,|X_n|$ occur with probability $0$.
Hence the maximizer
$W=\arg\max_{i\in [n]} |X_i|$
is almost surely unique and takes values in $\{1,\ldots,n\}$. By the
exchangeability of $\{X_1,\ldots,X_n\}$, we have for every $i$,
\[
\mathbb{P}(W=i)=\frac{1}{n}.
\]
Moreover, conditional on $S=s=(x_1,\ldots,x_n)$, the value of $W$ is deterministic:
$\mathbb{P}(W=i\mid S=s)=\mathbf{1}\{i=\arg\max_{j\in [n]} |x_j|\}$, which
is $\{0,1\}$-valued and
$\sum_{i=1}^n \mathbb{P}(W=i\mid S)=1$.

\paragraph{Kullback--Leibler mutual information:}
Write $P_{W\mid S}$ for the conditional law of $W$ given $S$ and $P_W$ for the marginal law. Then
\[
\begin{aligned}
I(W;S)&=\E\Bigl[ \DD_{\mathrm{KL}}\bigl(P_{W\mid S}\,\|\,P_W\bigr)\Bigr]
      =\E\biggl[\sum_{i=1}^n \mathbb{P}(W=i\mid S)\log\frac{\mathbb{P}(W=i\mid S)}{\mathbb{P}(W=i)}\biggr]\\
      &=\E\biggl[\sum_{i=1}^n \mathbb{P}(W=i\mid S)\log\frac{1}{1/n}\biggr]=\log n.
\end{aligned}
\]
\paragraph{R\'enyi mutual information.}
For $\alpha\neq 1$, 
\[
I_\alpha(W;S)=\DD_\alpha(P_{W,S}\,\|\,P_W\otimes P_S)
=\frac{1}{\alpha-1}\log \int \Bigl(\frac{dP_{W,S}}{d(P_W\otimes P_S)}\Bigr)^\alpha \, d(P_W\otimes P_S).
\]
Because $W$ is discrete and $S$ has density, the Radon--Nikodym derivative exists and equals
\[
\frac{dP_{W,S}}{d(P_W\otimes P_S)}(i,s)=\frac{\mathbb{P}(W=i\mid S=s)}{\mathbb{P}(W=i)}
= n\,\mathbb{P}(W=i\mid S=s).
\]
Hence, using $\mathbb{P}(W=i\mid S=s)\in\{0,1\}$ and $\sum_{i=1}^n \mathbb{P}(W=i\mid S=s)=1$,
\begin{align}\label{eq:p-moment}
\int \Bigl(\frac{dP_{W,S}}{d(P_W\otimes P_S)}\Bigr)^\alpha d(P_W\otimes P_S)
&=\E\Bigl[\sum_{i=1}^n \mathbb{P}(W=i)\bigl(n\mathbb{P}(W=i\mid S)\bigr)^\alpha\Bigr]\nonumber\\
&=\E\Bigl[\frac{1}{n}\sum_{i=1}^n n^\alpha \mathbb{P}(W=i\mid S)\Bigr]=\E\bigl[n^{\alpha-1}\bigr]
= n^{\alpha-1}.
\end{align}
Therefore, $I_\alpha(W;S)=\frac{1}{\alpha-1}\log\bigl(n^{\alpha-1}\bigr)=\log n$ for $\alpha\neq 1$. Taking $\alpha\to 1$ yields $I(W;S)=\lim_{\alpha\to 1} I_\alpha(W;S)=\log n$.
\end{proof}

\subsection{Proof of Theorem \ref{Dudley}}
\begin{proof}
The proof proceeds in four steps.\\
\medskip
\noindent\textbf{Step 1: Discretization via $\varepsilon$-nets.}
If $\abs{T}=1$, the result is trivial. Assume $\abs{T}\ge 2$. For $k\ge 0$, set
$\varepsilon_k:=2^{-k}e(T)$. Since $T$ is finite, there exists
\[
\Delta:=\min\{d(s,t):s,t\in T,\ s\neq t\}>0.
\]
Choose $K$ so large that $\varepsilon_K<\Delta/2$. Then every $\varepsilon_K$-ball contains at most one point of $T$, so every $\varepsilon_K$-net is all of $T$ and $T_K=T$.

Fix $k\in\{1,\dots,K\}$ and $u\in T_k$. Since $T_{k-1}$ is an
$\varepsilon_{k-1}$-net of $T$ and $u\in T_k\subseteq T$, there exists
$p_k(u)\in T_{k-1}$ such that
\[
d\bigl(u,p_k(u)\bigr)\le \varepsilon_{k-1}.
\]
This proves the existence of the parent map $p_k$.

\medskip
\noindent\textbf{Step 2: Apply the maximal inequality.}
Next, for each $k$, the set of possible $k$th increments is contained in
$
\bigl\{X_u-X_{p_k(u)}:u\in T_k\bigr\}$. For each $k\ge 1$ and $u\in T_k$, set $Y_{k,u}:=X_u-X_{p_k(u)}$. Then
\[
\norm{Y_{k,u}}_{\psi_\theta}\le C\,d\paren{u,p_k(u)}\le 2C\varepsilon_k.
\]
Moreover,
$
\sup_{t\in T}\abs{X_{\pi_k(t)}-X_{\pi_{k-1}(t)}}
\le
\max_{u\in T_k}\abs{Y_{k,u}}$. Applying Lemma~\ref{pp-MaximalSW} at scale $k$ gives
\[
\E\sup_{t\in T}\abs{X_{\pi_k(t)}-X_{\pi_{k-1}(t)}}
\le
2C\varepsilon_k\,
\psi_\theta^{-1}\!\paren{\psi_\theta(x_\theta)+N(T,d,\varepsilon_k)}.
\]
For $k\ge 1$ we have $N(T,d,\varepsilon_k)\ge 2$, so by the definition of $K_\theta$,
\[
\psi_\theta^{-1}\!\paren{\psi_\theta(x_\theta)+N(T,d,\varepsilon_k)}
\le
K_\theta \bracks{\log N(T,d,\varepsilon_k)}^{1/\theta}.
\]
Hence
\[
\E \sup_{t\in T} X_t
\le
2CK_\theta
\sum_{k=1}^K
\varepsilon_k \bracks{\log N(T,d,\varepsilon_k)}^{1/\theta}.
\]
\medskip
\noindent\textbf{Step 3: Converting to an integral.}
Note that $\varepsilon_k=2(\varepsilon_k-\varepsilon_{k+1})$ and that $\varepsilon\mapsto N(T,d,\varepsilon)$ is nonincreasing. Therefore
\[
\varepsilon_k \bracks{\log N(T,d,\varepsilon_k)}^{1/\theta}
\le
2\int_{\varepsilon_{k+1}}^{\varepsilon_k}
\bracks{\log N(T,d,\varepsilon)}^{1/\theta}\,d\varepsilon.
\]
Summing over $k$ yields the desired bound
\[
\E \sup_{t\in T} X_t
\le
4CK_\theta
\int_0^{e(T)}
\bracks{\log N(T,d,\varepsilon)}^{1/\theta}\,d\varepsilon
\le
4CK_\theta
\int_0^\infty
\bracks{\log N(T,d,\varepsilon)}^{1/\theta}\,d\varepsilon.
\]
\end{proof}

\subsection{Proof of Corollary \ref{Dudley1}}\label{se:Dudley1}
\begin{proof}
By separability, there exists a countable dense set $T_0=\{t_1,t_2,\dots\}\subseteq T$ such that
\[
\sup_{t\in T} X_t = \sup_{t\in T_0} X_t
\qquad\text{a.s.}
\]
For $m\ge 1$, let
$T^{(m)}:=\{t_1,\dots,t_m\},~
M_m:=\sup_{t\in T^{(m)}}X_t$. Then $M_m\uparrow \sup_{t\in T_0}X_t$ a.s. Since $T^{(m)}\subseteq T$,
\[
N(T^{(m)},d,\varepsilon)\le N(T,d,\varepsilon),
\qquad \varepsilon>0.
\]
Applying Theorem~\ref{Dudley} to the finite set $T^{(m)}$,
\[
\E M_m
\le
4CK_\theta
\int_0^\infty
\bracks{\log N\paren{T^{(m)},d,\varepsilon}}^{1/\theta}\,d\varepsilon
\le
4CK_\theta
\int_0^\infty
\bracks{\log N(T,d,\varepsilon)}^{1/\theta}\,d\varepsilon.
\]
By the monotone convergence theorem,
\[
\E\sup_{t\in T} X_t
=
\E\sup_{t\in T_0} X_t
=
\lim_{m\to\infty}\E M_m
\le
4CK_\theta
\int_0^\infty
\bracks{\log N(T,d,\varepsilon)}^{1/\theta}\,d\varepsilon.
\]
\end{proof}

\section{Proofs for Section \ref{sec:pre}}\label{proofse3}

This section collects proofs for the main results in Section \ref{sec:pre}.
\subsection{Proof of Lemma \ref{lem:key1}}
\begin{proof} 
We first establish the result for the case $1<\alpha\le 2$, and then treat the case $\alpha>2$.

\paragraph{Case $1<\alpha\le 2$.}
When $\theta\ge 1$ and $A=1$, the claim follows directly by combining
\cite[Proposition~1]{chu2023unified} with the monotonicity of \Renyi divergence stated in Lemma~\ref{Renyi}.

So we only need to bound the gap between $A>1$ and $A=1$ case:
\[
b_1=\E_Q \frac{dP}{dQ} \log^{1/\theta} \bigl(\frac{d P}{d Q} +A \bigr )-\E_Q \frac{dP}{dQ} \log^{1/\theta} \bigl(\frac{d P}{d Q} +1 \bigr ).
\]
Note that
\[
x\log^{1/\theta}(x+A)-x\log^{1/\theta}(x+1)\leq x \log^{1/\theta}\bigl ( 1+\frac{A-1}{x+1}\bigr ) \leq x (A-1)^{1/\theta}.
\]
Let $x=dP/dQ$. Taking expectation with respect to $Q$ yields $b_1\leq (A-1)^{1/\theta}$, which completes this case.

We therefore focus on the case $0<\theta<1$. Using the elementary inequality
$(x+y)^k \le x^k+y^k$ for $x,y\ge 0$ and $0<k\le 1$, we obtain
\[
\begin{aligned}
\DD_{f_\theta}(P\|Q)
&= \mathbb{E}_P\!\left[
\frac{1}{\alpha-1}
\log\!\left(\left|\frac{dP}{dQ}\right|+A\right)^{\alpha-1}
\right]^{1/\theta} \le
\mathbb{E}_P\!\left[
\frac{1}{\alpha-1}
\log\!\left(
\left|\frac{dP}{dQ}\right|^{\alpha-1}
+ A^{\alpha-1}
\right)
\right]^{1/\theta}.
\end{aligned}
\]

By Lemma~\ref{lemma:concave} in Appendix~\ref{app:lemmas}, the $\log^{1/\theta}(x+A^{\alpha-1})$ is concave on $[0,\infty)$ provided that
\[
A^{\alpha-1} \;\ge\; \exp\!\left(\frac{1}{\theta}-1\right).
\]
Applying Jensen's inequality then yields
\[
\DD_{f_\theta}(P\|Q)
\le
\left[
\frac{1}{\alpha-1}
\log\!\left(
\mathbb{E}_P\!\left|\frac{dP}{dQ}\right|^{\alpha-1}
+ A^{\alpha-1}
\right)
\right]^{1/\theta}.
\]

To separate the logarithmic terms, for $x\ge 1$ and $y\ge 0$,
$x+y \le x(1+y)$, 
\[
\log(x+y) \le \log x + \log(1+y).
\]
Thus it suffices to verify that
\[
\mathbb{E}_P\!\left|\frac{dP}{dQ}\right|^{\alpha-1} \;\ge\; 1.
\]
Indeed, writing $L:=\frac{dP}{dQ}\ge 0$ (defined $Q$-a.s.), Jensen's inequality gives
\[
\mathbb{E}_P[L^{\alpha-1}]
=\mathbb{E}_Q[L^\alpha]
\;\ge\;
\bigl(\mathbb{E}_Q L\bigr)^\alpha
=1.
\]
Consequently,
\[
\DD_{f_\theta}(P\|Q)
\le
\left[
\DD_\alpha(P\|Q)
+
\frac{1}{\alpha-1}\log(1+A^{\alpha-1})
\right]^{1/\theta},
\]
which establishes the desired bound for $1<\alpha\le 2$.

\paragraph{Case $\alpha>2$.}
For $\alpha>2$, we invoke the monotonicity of \Renyi divergence,
$\DD_2(P\|Q)\le \DD_\alpha(P\|Q)$, to obtain
\[
\DD_{f_\theta}(P\|Q)
\;\le\;
\inf_{\alpha>2}
\left\{
\bigl(\DD_\alpha(P\|Q)+C_{\alpha,\theta}\bigr)^{1/\theta}
\right\}
=
\bigl(\DD_2(P\|Q)+C_{2,\theta}\bigr)^{1/\theta}.
\]
\end{proof}

\subsection{Proof of Lemma \ref{de}}
\begin{proof}
Since
$\E_\mu r = \E_\nu \frac{d \mu}{d \nu} r$. Now set $x=\frac{d \mu}{d \nu}$, $y=r$ in Lemma \ref{young}, taking expectation gives
\[
\E_\mu r = \E_\nu \frac{d \mu}{d \nu} r\leq 2^{\frac{1}{\theta}} \E_\nu \{\frac{d \mu}{d \nu} [\log(\frac{d \mu}{d \nu} + A)]^{\frac{1}{\theta}}\} + \E_\nu\exp(y^\theta),
\]
and the definition of $\DD_{f_\theta}(\mu \| \nu)$ shows
\[
\E_\mu r \leq 2^\I\DD_{f_\theta}(\mu \| \nu) + \E_\nu \exp(r^\theta).
\]
Finally, Lemma~\ref{lem:key1} gives the R\'enyi upper bound
$\E_\mu r \leq 2^\I[\DD_\al(\mu \| \nu)+ C_{\alpha,\theta}]^\I+ \E_\nu \exp(r^\theta)$.
\end{proof}



\section{Proofs for Section \ref{sec:gen_bounds}}\label{se:PCT}
\subsection{Mutual information for randomized maximum selector}
Let $S=\{X_1,\ldots,X_n\}$ with i.i.d.\ Weibull\((\theta)\) coordinates
(\(\theta<1\)), so $\mathbb{P}(X>x)=e^{-x^\theta}$. The selector \(W\in[n]\) is
defined by
\[
\mathbb{P}(W=y\mid S=x)=:p(y\mid x)=
\begin{cases}
\epsilon+\dfrac{1-\epsilon}{n}, & \text{if } y=\arg\max_{i\in[n]} x_i,\\[6pt]
\dfrac{1-\epsilon}{n}, & \text{otherwise,}
\end{cases}
\]
where the argmax is a.s.\ unique by the continuity of $\{X_i\}$. 

By symmetry, the marginal of \(W\) is uniform:
\[
p(y)=\mathbb{P}(W=y)=\frac{1}{n},\qquad y\in[n].
\]
Write
\[
p_m:=\epsilon+\frac{1-\epsilon}{n}=\frac{1+\epsilon(n-1)}{n},
\qquad
p_o:=\frac{1-\epsilon}{n}.
\]
Then, for any realization \(x\), exactly one index is the maximizer and
\[
\sum_{y=1}^n p(y\mid x)\log\frac{p(y\mid x)}{p(y)}
=\sum_{y=1}^n p(y\mid x)\log\bigl(n\,p(y\mid x)\bigr)
= p_m\log(n p_m) + (n-1)\,p_o\log(n p_o),
\]
which is constant in \(x\). Therefore,
\begin{align*}
I(W;S)
&=\int p(x)\sum_{y=1}^n p(y\mid x)\log\frac{p(y\mid x)}{p(y)}dx= p_m\log(n p_m) + (n-1)\,p_o\log(n p_o)\\
&=\frac{1}{n}\Bigl[(1+\epsilon(n-1))\,\log\bigl(1+\epsilon(n-1)\bigr)
+ (n-1)\,(1-\epsilon)\,\log(1-\epsilon)\Bigr].
\end{align*}

\subsection{Proof of Theorem \ref{infob}}
\begin{proof}
Without loss of generality, we assume $\max_{i\in[n]}\|X_i\|_{\psi_\theta}=1$. By changing measure,
\[
\mathbb{E}_{P_{W,S}}(|X_W|) = \mathbb{E}_{P_W \otimes P_S} \left(|X_W| \frac{d P_{W,S}}{d P_W \otimes P_S} \right).
\]
Let $\mu = P_{W,S}$, $\nu=P_{W} \otimes P_{S}$ and $r=|X_W|$. Use Lemma \ref{de} (decorrelation lemma)
\[
\E_{P_{W,S}}(|X_W|) \leq 2^\I I_{f_\theta}(W,S) + \E_{P_W \otimes P_S} \exp(|X_W|^\theta).
\]
By sub-Weibullity,
$\E_{P_W \otimes P_S} \exp |X_W|^\theta = \E_{w \sim P_W} \E_{P_S} \exp |X_w|^\theta \leq 2$. Then, if $A$ satisfies the conditions of Lemma~\ref{lem:key1}, then
\[
\mathbb{E}|X_W|
\;\le\;
\max_{i\in[n]}\|X_i\|_{\psi_\theta}\,
\Bigl(2^{1/\theta}\bigl(I_\alpha(W;S)+C_{\alpha,\theta}\bigr)^{1/\theta}+2\Bigr).
\]
\end{proof}

\subsection{Proof of Theorem \ref{chaining}}
\begin{proof}
Without loss of generality, assume $C=1$. \\
\textbf{Step 1: The partitions get finer.}
Assume that we are given increasing sequence partitions $\{\mathcal P_k\}_{k\ge 0}$ such that $\mathcal P_0=\{T\}$, and $\mathcal P_k$ is a $\epsilon_k$-partition for $k\ge1$ with $\epsilon_k := e(T)2^{-k}$. Let $[W]_k\in\mathcal P_k$ be the cell-valued random variable that contains $W$ with 
$$
T=[W]_0 \supset[W]_1 \supset[W]_2 \supset[W]_3 \supset \cdots.
$$
For each $k \ge 0$, let $W_k:=t_{[W]_k}$ be a designated representative point of $[W]_k \in \mathcal{P}_k$. 

For each cell $A \in \mathcal{P}_k$, choose a representative $t_A \in T$ such that 
$$A \subset \mathcal{B}_d\left(t_A, \varepsilon_k\right).$$
Moreover, choose these representatives hierarchically so that whenever $A \in \mathcal{P}_k$ is contained in its parent cell $B \in \mathcal{P}_{k-1}$, one has $t_A \in B$. Hence, for every such parent-child pair, it gives a parent-child distance 
\begin{equation}\label{eq: parent-child}
    d\left(t_A, t_B\right) \leq \varepsilon_{k-1} .
\end{equation}

\begin{figure}[t]
\centering 
\includegraphics[scale=0.078]{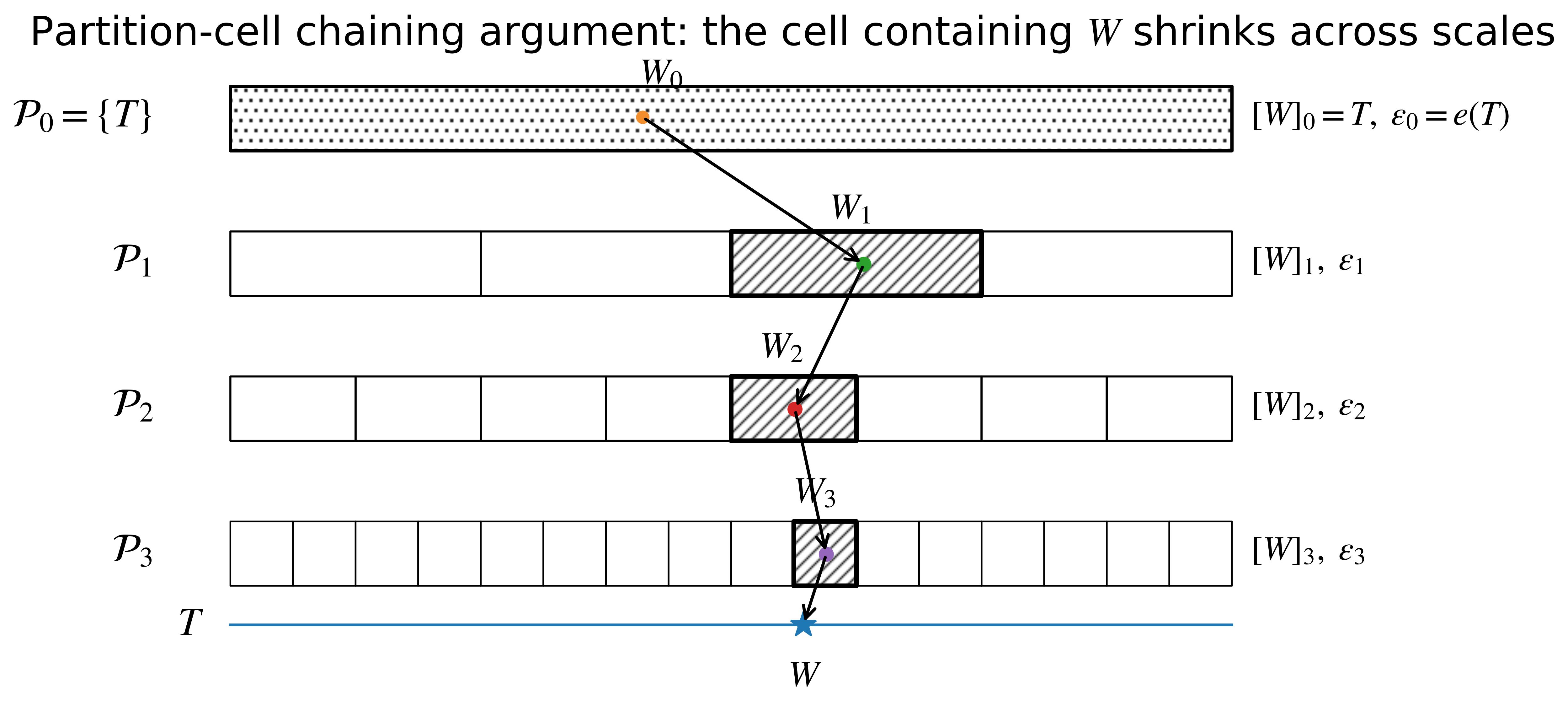}
\caption{Partition-cell chaining argument.}
\end{figure}

\textbf{Step 2: Telescope along the representatives for finite \( |T| \).}
For finite $|T|$, there exists a large $m$ such that each cell in $\mathcal{P}_m$ contains at most one point of $T$. Hence $W_{m}=W$. Pick an arbitrary anchor point $t_0 \in [W]_0=T$. Then
\[
X_W = X_{t_0} + \sum_{k=1}^{m}(X_{W_k}-X_{W_{k-1}}).
\]
For $k \geq 1$, by the hierarchical choice of representatives and $[W]_k \subset [W]_{k-1}$, \eqref{eq: parent-child} implies
\[
    d(W_{k}, W_{{k-1}})\le \epsilon_{k-1} = e(T)2^{-k+1}.
\]
Let $\delta_k = e(T)2^{-(k-1)}$. Because $\E[X_{t_0}] = 0$, the linearity of expectation shows,
\[
    \E[X_W] = \sum_{k=1}^m \delta_k \, \E \left[ \frac{X_{W_{k}} - X_{W_{{k-1}}}}{\delta_k} \right]\le \sum_{k=1}^m \delta_k \, \E \left| \frac{X_{W_{k}} - X_{W_{{k-1}}}}{\delta_k} \right|.
\]
To bound the expectation of the $k$-th term, set
$U_k:=(W_k,W_{k-1})$ and
$Z_k=|X_{W_k}-X_{W_{k-1}}|/\delta_k$. Applying the decorrelation lemma
(Lemma~\ref{de}) with $\mu=P_{U_k,S}$ and
$\nu=P_{U_k}\otimes P_S$ gives
\[
  \E_{P_{U_k,S}}[Z_k]
  \le
  2^{1/\theta} I_{f_\theta}(U_k; S)
  + \E_{P_{U_k}\otimes P_S} \left[ \exp(Z_k^\theta) \right].
\]
The pair $U_k$ is a deterministic function of the cell $[W]_k$, because
$W_k=t_{[W]_k}$ and $W_{k-1}$ is the representative of the parent cell.
Hence, by data processing for $f$-divergences,
$I_{f_\theta}(U_k;S)\le I_{f_\theta}([W]_k;S)$. Under the product measure
$P_{U_k}\otimes P_S$, condition on $U_k=(u,v)$. The indices $u,v$ are then
fixed and satisfy $d(u,v)\le\delta_k$, so the $(\theta,1)$-sub-Weibull
increment assumption gives
\[
    \E_{P_S}\exp\left(\frac{|X_u-X_v|^\theta}{\delta_k^\theta}\right)\le 2.
\]
Integrating over $P_{U_k}$ yields
$\E_{P_{U_k}\otimes P_S}[\exp(Z_k^\theta)]\le 2$. Substituting this back into
the sum gives
\[
    \E[X_W] \le \sum_{k=1}^m e(T)2^{-(k-1)} \left( 2^{1/\theta} I_{f_\theta}([W]_k; S) + 2 \right),
\]
which proves the theorem for finite $T$.

\textbf{Step 3: Extension to separable processes.}
For $m\ge 1$, define the $m$-th hierarchical approximation of $W$ by
$W_m:=t_{[W]_m}$. The representatives are chosen along the separability
approximation described in Definition~\ref{def:separable}. Since
$W\in[W]_m\subset B_d(W_m,e(T)2^{-m})$, we have
\[
d(W_m,W)\le e(T)2^{-m}\xrightarrow[m\to\infty]{}0
\quad\text{a.s.}
\]
Almost-sure sample-path continuity
 implies
\[
X_{W_m}\xrightarrow[m\to\infty]{}X_W
\quad\text{a.s.}
\]
To justify dominated convergence $\E[X_W] = \lim_{m \to \infty} \E[X_{W_m}]$, fix any $u_0 \in T$. Then
$$
\left|X_{W_m}\right| \leq\left|X_{u_0}\right|+\sup _{t \in T}\left|X_t-X_{u_0}\right| .
$$
The first term is integrable because a sub-Weibull random variable has finite first moment. For the second term,
$$
\sup _{t \in T}\left|X_t-X_{u_0}\right| \leq \sup _{t \in T}\left(X_t-X_{u_0}\right)+\sup _{t \in T}\left(X_{u_0}-X_t\right) .
$$
Both processes
$\left\{X_t-X_{u_0}\right\}_{t \in T}$ and $\left\{X_{u_0}-X_t\right\}_{t \in T}
$ are again separable $(\theta, C)$-sub-Weibull processes, so Corollary \ref{Dudley1} implies that each supremum has finite expectation. The dominated convergence theorem justifies the passage to the limit:
\[
    \E[X_W] = \lim_{m \to \infty} \E[X_{W_m}] \le \sum_{k=1}^\infty e(T)2^{-(k-1)} \left( 2^{1/\theta} I_{f_\theta}([W]_k; S) + 2 \right).
\]
	The R\'enyi refinement follows immediately by applying Lemma~\ref{lem:key1} to bound the $I_{f_\theta}$ term via R\'enyi mutual information.
\end{proof}

\section{Proofs in Section \ref{sec:rlhf}}\label{se:API}
\subsection{Proof of Lemma \ref{renyisol}}
\begin{proof}
Fix $x\in\mathcal X$ and write $P_0:=\pi_0(\cdot\mid x),~
r(y):=r(x,y)$. For any feasible policy $\pi(\cdot\mid x)\ll P_0$, let
$
q(y):=\frac{d\pi(\cdot\mid x)}{dP_0}(y).
$
The optimization problem is equivalent to
\[
\sup_{q\ge0}
\int_{\mathcal Y} r(y)q(y)\,P_0(dy)
\]
subject to $\int_{\mathcal Y}q\,dP_0=1,~\int_{\mathcal Y}q^\alpha\,dP_0\le \rho$, where
$\rho:=\exp((\alpha-1)\epsilon)$.

Since $r\in L^{\alpha/(\alpha-1)}(P_0)$ and $q\in L^\alpha(P_0)$ for every
feasible $q$, the objective is finite by H\"older's inequality. The feasible
set is convex, and the objective is linear. Moreover, since $\epsilon>0$,
we have $\rho>1$, and $q\equiv1$ satisfies
\[
\int q\,dP_0=1,
\qquad
\int q^\alpha\,dP_0=1<\rho.
\]
Thus Slater's condition holds, so the KKT conditions are necessary for
optimality. Let
\[
r_x^\star:=\operatorname*{ess\,sup}_{y\sim P_0}r(y),
\qquad
M_x:=\{y:r(y)=r_x^\star\},
\qquad
m_x:=P_0(M_x).
\]

We first characterize when a policy supported entirely on $M_x$ is feasible.
Suppose $q$ is supported on $M_x$. Then
\[
\int_{M_x}q\,dP_0=1.
\]
If $m_x=0$, no such density exists. If $m_x>0$, Jensen's inequality under the
conditional probability measure $P_0(\cdot\mid M_x)$ gives
\[
\int q^\alpha\,dP_0
=
\int_{M_x}q^\alpha\,dP_0
=
m_x\int_{M_x}q^\alpha\,dP_0(\cdot\mid M_x)
\ge
m_x
\left(
\int_{M_x}q\,dP_0(\cdot\mid M_x)
\right)^\alpha.
\]
Since $\int_{M_x}q\,dP_0(\cdot\mid M_x)
=
\frac1{m_x}\int_{M_x}q\,dP_0
=
\frac1{m_x}$, we get
\[
\int q^\alpha\,dP_0
\ge
m_x^{1-\alpha}.
\]
Equality is attained by $
q_M(y)=\frac{\mathbf 1_{M_x}(y)}{m_x}$. Therefore, the minimal R\'enyi divergence among policies supported on $M_x$ is $\frac1{\alpha-1}\log m_x^{1-\alpha}
=
-\log m_x$. Hence, if
\[
\epsilon\ge -\log m_x,
\]
then $q_M$ is feasible. Since $q_M$ puts all its mass on the set where
$r(y)=r_x^\star$, it achieves the largest possible objective value
\[
\int r(y)q_M(y)\,P_0(dy)=r_x^\star.
\]
No policy can achieve objective value larger than $r_x^\star$ by the definition
of essential supremum. Thus $q_M$ is optimal. This proves the first case.

It remains to consider the case
\[
\epsilon<-\log m_x.
\]
In this case, no feasible policy can be supported entirely on $M_x$.

Let $t\in\mathbb R$ be the multiplier for
the normalization constraint, and let $\lambda\ge0$ be the multiplier for the
R\'enyi-power constraint. The Lagrangian is
\[
\mathcal L(q,t,\lambda)
=
\int r(y)q(y)\,dP_0
-
t\left(\int q\,dP_0-1\right)
-
\lambda\left(\int q^\alpha\,dP_0-\rho\right).
\]
Equivalently,
\[
\mathcal L(q,t,\lambda)
=
t+\lambda\rho
+
\int
\left[
(r(y)-t)q(y)-\lambda q(y)^\alpha
\right]P_0(dy).
\]
The non-negativity constraint $q\ge0$ is treated as a domain constraint.

The complementary slackness condition is
\[
\lambda
\left(
\int q^*(y)^\alpha\,P_0(dy)-\rho
\right)
=
0.
\]

We claim that in the present case one must have
\[
\lambda>0.
\]
Suppose instead that $\lambda=0$. Then the pointwise part of the Lagrangian is
\[
(r(y)-t)q(y),
\qquad q(y)\ge0.
\]
For the supremum over $q(y)\ge0$ to be finite, it is necessary that
\[
r(y)-t\le0
\qquad P_0\text{-a.s.}
\]
Thus, $t\ge r_x^\star$. On the other hand, since $q^*$ is feasible,
\[
\int q^*\,dP_0=1,
\]
so $q^*>0$ on a set of positive $P_0$-measure. On such a set, the pointwise
linear function $(r(y)-t)q$ can be maximized at a positive value of $q$ only if
\[
r(y)-t=0.
\]
Therefore $q^*$ must be supported on
\[
M_x=\{y:r(y)=r_x^\star\}.
\]
But in the case $\epsilon<-\log m_x$, no feasible policy can be supported
entirely on $M_x$. This is a contradiction. Hence, $\lambda>0$.

By complementary slackness, $\lambda>0$ implies
\[
\int q^*(y)^\alpha\,P_0(dy)=\rho.
\]
Equivalently, $\mathcal D_\alpha(\pi^*(\cdot\mid x)\|\pi_0(\cdot\mid x))
=
\epsilon$. Thus the optimum is attained on the boundary of the R\'enyi ball.

It remains to derive the form of $q^*$. For fixed $t$ and $\lambda>0$, the
Lagrangian is separable in $y$. For each $y$, one has to solve
\[
\max_{q\ge0}
\left\{
(r(y)-t)q-\lambda q^\alpha
\right\}.
\]
Define
$
\phi_y(q):=(r(y)-t)q-\lambda q^\alpha$ for $q\ge0$. Since $\alpha>1$ and $\lambda>0$, $\phi_y$ is strictly concave.

If $r(y)\le t$, then for every $q>0$,
\[
\phi_y'(q)
=
r(y)-t-\alpha\lambda q^{\alpha-1}
<0.
\]
Thus the maximizer is $
q^*(y)=0$. If $r(y)>t$, then
\[
\phi_y'(0+)=r(y)-t>0,
\]
so the maximizer is an interior point and satisfies
$
r(y)-t-\alpha\lambda q(y)^{\alpha-1}=0$. Therefore, $q^*(y)
=
\left(
\frac{r(y)-t}{\alpha\lambda}
\right)^{\frac1{\alpha-1}}$. Combining the two cases above yields
\[
q^*(y)
=
\left(
\frac{r(y)-t}{\alpha\lambda}
\right)_+^{\frac1{\alpha-1}}
\qquad P_0\text{-a.s.}
\]

Let $\beta:=\frac1{\alpha-1}$. Then,
$
q^*(y)
=
(\alpha\lambda)^{-\beta}(r(y)-t)_+^\beta$. The multiplicative constant is determined by the constraint $\int q^*\,dP_0=1$. Hence
\[
q^*(y)
=
\frac{(r(y)-t)_+^\beta}
{\int_{\mathcal Y}(r(y')-t)_+^\beta\,P_0(dy')}.
\]
Returning to the original notation,
\[
\frac{d\pi^*(\cdot\mid x)}{d\pi_0(\cdot\mid x)}(y)
=
\frac{(r(x,y)-t)_+^{\frac1{\alpha-1}}}
{\mathbb E_{Y\sim\pi_0(\cdot\mid x)}
[(r(x,Y)-t)_+^{\frac1{\alpha-1}}]}.
\]
The threshold $t$ is chosen so that $\mathcal D_\alpha(\pi^*(\cdot\mid x)\|\pi_0(\cdot\mid x))
=
\epsilon$. This completes the proof.
\end{proof}

\subsection{Proof of Theorem \ref{thm:rlhf}}
\begin{proof}
Fix $x\in\mathcal X$ and abbreviate $\pi(\cdot\mid x)$ and $\pi_0(\cdot\mid x)$
by $\pi$ and $\pi_0$, respectively. Let
\[
\bar r(x,y):= r(x,y)-\E_{Y\sim\pi_0}r(x,Y),
\qquad y\in\mathcal Y.
\]
Define the nonnegative random variable
$Z:=\frac{|\bar r(Y,x)|}{C}$ with $Y\sim\pi_0(\cdot\mid x)$. By the assumption $\|\bar r(x,Y)\|_{\psi_\theta}\le C$, under $Y\sim\pi_0(\cdot\mid x)$, we have
\[
\E_{\pi_0}\exp(Z^\theta)
=
\E_{\pi_0}\exp\!\left(\frac{|\bar r(Y)|^\theta}{C^\theta}\right)
\le 2.
\]

We now apply the decorrelation lemma (Lemma~\ref{de}) with $\mu=\pi$, $\nu=\pi_0$, and
the measurable function $r(\cdot):=|\bar r(\cdot,x)|/C$. This yields
\[
\E_{\pi}\!\left[\frac{|\bar r(Y)|}{C}\right]
\;\le\;
2^{1/\theta}\,\DD_{f_\theta}(\pi\|\pi_0)
\;+\;
\E_{\pi_0}\exp\!\left(\frac{|\bar r(Y)|^\theta}{C^\theta}\right)
\;\le\;
2^{1/\theta}\,\DD_{f_\theta}(\pi\|\pi_0)+2.
\]
Multiplying both sides by $C$ gives
\[
\E_{\pi}|\bar r(Y)|
\;\le\;
C\bigl(2^{1/\theta}\,\DD_{f_\theta}(\pi\|\pi_0)+2\bigr).
\]
Since $\E_\pi \bar r(Y)=\E_\pi r(x,Y)-\E_{\pi_0}r(x,Y)$, we conclude by
$|\E_\pi \bar r|\le \E_\pi|\bar r|$ that
\[
\bigl|\E_\pi r(x,Y)-\E_{\pi_0}r(x,Y)\bigr|
\;\le\;
C\bigl(2^{1/\theta}\,\DD_{f_\theta}(\pi\|\pi_0)+2\bigr).
\]
Finally, invoking Lemma~\ref{lem:key1} to upper bound
$\DD_{f_\theta}(\pi\|\pi_0)$ in terms of $\DD_\alpha(\pi\|\pi_0)$ yields the
displayed R\'enyi-based bound in the theorem.

For $\pi^*$ solving \eqref{optRENYI}, we have
$\DD_\alpha(\pi^*\|\pi_0)\le\epsilon$ for each $x$. Substituting this into the
preceding bound (and averaging over $x$ if desired) gives the stated guarantee
for $\pi^*$.
\end{proof}

\subsection{Proof of Theorem \ref{renyiregret}}

\begin{proof}
By Assumption \ref{ass:reward-structure}, we need to bound 
\[
\DD_\alpha (r_n(x)\|r(x)).
\]
This reduces to the following lemma by setting $r_n(x)\sim\pi_n=:\nu$ and $r(x)\sim\pi_0=:\mu$.
\begin{lemma}[\Renyi divergence of the maximum of i.i.d.\ draws]\label{DDD} 
Let $\mu$ and $\nu$ be any two univariate probability distributions. Assume $\alpha>1$. Let $X$ and $\{X_i\}_{i=1}^n \overset{\iid}{\sim} \mu$ be random variables, and put $R_n=\max_{i\in [n]} X_i \sim \nu$. Then
\[
D_\alpha(\nu \| \mu) \leq \frac{1}{\alpha-1}\log \left( \frac{n^\alpha}{\alpha(n-1)+1} \right)
\]
and $D_{KL}(\nu \| \mu) \leq \log n-\frac{n-1}{n}$.
\end{lemma}

\begin{proof}
The KL case is given by Theorem 1 in \cite{mroueh2025information} and can also be proved by taking the limit $\alpha \to 1$, so here we prove the \Renyi case: $\alpha>1$.

Let $Q$ be the left-continuous quantile function of $X$. For $U$ and $\{U_i\}_{i=1}^n \overset{\iid}{\sim}\mathrm{U}[0,1]$,
\[
R\stackrel{d}=Q(U),
\qquad
R_n\stackrel{d}=Q(U_{(n)}),
\]
where $U_{(n)}:=\max_{1\le i\le n}U_i$.
By data processing inequality for R\'enyi divergence,
\[
D_\alpha(R_n\|R)\le D_\alpha(U_{(n)}\|U).
\]
Now $U_{(n)}$ has density $n u^{n-1}$ on $[0,1]$, so
\[
D_\alpha(R_n\|R)\le D_\alpha(U_{(n)}\|U)
=
\frac{1}{\alpha-1}\log\int_0^1 \paren{n u^{n-1}}^\alpha\,du
=
\frac{1}{\alpha-1}\log\frac{n^\alpha}{\alpha(n-1)+1}.
\]
By L'Hôpital's Rule, $\lim_{\alpha \to 1}\frac{1}{\alpha-1}\log\frac{n^\alpha}{\alpha(n-1)+1}=\log n-\frac{n-1}{n}$ and we have $D_{KL}(\nu \| \mu) \leq \log n-\frac{n-1}{n}$.
\end{proof}

Follows from Assumption~\ref{ass:reward-structure} and another application of data processing:
\[
D_\alpha(\pi_n\|\pi_0)\le D_\alpha(R_n\|R).
\]
Now, in this setting, since $\frac{1}{\alpha-1}\log \left( \frac{n^\alpha}{\alpha(n-1)+1} \right)\leq \log n$ for $\alpha>1$ and $n\ge 1$, we have 
\[
\DD_\alpha(\pi_n \|\pi_0) \leq \frac{1}{\alpha-1}\log \left( \frac{n^\alpha}{\alpha(n-1)+1} \right)\leq \log n
\]
So if $n\leq \exp(\varepsilon)$, then $\DD_\alpha(\pi_n\|\pi_0)\leq \varepsilon$.

Applying the decorrelation lemma (Lemma \ref{de}) to $\bar r(x,Y):= |r(x,Y)-\E_{Y\sim\pi_0}r(x,Y)|/C,~~y\in\mathcal Y$, similar to the proof of Theorem \ref{thm:rlhf}, we have
\[
\mathbb{E}_{\pi_n(\cdot|x)} r - \mathbb{E}_{\pi_0(\cdot|x)} r \le |\mathbb{E}_{\pi_n(\cdot|x)} r - \mathbb{E}_{\pi_0(\cdot|x)} r|
\leq C \left( 2^{1/\theta} \left[ \DD_\alpha(\pi_n \| \pi_0) +C_{\alpha,\theta}\right]^{1/\theta} + 2 \right).
\]
from $\DD_\alpha(\pi_n\|\pi_0)\leq \varepsilon$. This completes the proof.
\end{proof}

\section{Proofs of Useful Lemmas and Formulas}\label{app:lemmas}

This section presents several lemmas used in the proofs of the main results. We begin by showing that the shifted-log function is convex.

\begin{lemma}\label{lemma:convex}
For $\theta>0$, the function $x \mapsto x\log^\theta(x + A),~A \geq 1$ is convex over $x>0$.
\end{lemma}

\begin{proof}
Let $L = \log(x + A)$, and define $f(x) := x\, L^{\theta}$.
The first derivative is
\$
f'(x) = L^{\theta} + x \theta L^{\theta - 1} \frac{1}{x + A}
= L^{\theta - 1} \left( L + \frac{\theta x}{x + A} \right)=L^{\theta - 1} \left( L + \theta\left[1-\frac{A}{x + A}\right] \right),
\$
while the second derivative is 
\begin{align*}
f''(x)
    &= (\theta - 1)L^{\theta - 2}\frac{1}{x + A}\left( L + \frac{\theta x}{x + A} \right)
+ L^{\theta - 1}\left(\frac{1}{x+A} + \frac{\theta A}{(x + A)^2} \right)\\
    &=\theta\frac{(\log(x+A))^{\theta-2}}{(x+A)^2}\Big[(x+2A)\log(x+A)+(\theta-1)x\Big].
\end{align*}
Since $\theta>0$, $(x+A)^2>0$, and for $x>0, A\ge 1$, we have $\log(x+A)\ge 0$, so $\log^{\theta-2}(x+A)>0$ whenever it is defined. It remains to prove that the bracketed term is positive. Define
$$B(x)=(x+2A)\log(x+A)+(\theta-1)x.$$
The worst case occurs as $\theta\to0^+$, so it suffices to show
$$(x+2A)\log(x+A)-x>0\quad\text{for all}~x>0.$$
Define $h(x)=(x+2A)\log(x+A)-x$. Then $h(0)=2A\log A\ge 0$ (since $A\ge1$), and
$$h^{\prime}(x)=\log(x+A)+\frac{x+2A}{x+A}-1=\log(x+A)+\frac A{x+A}>0\quad\mathrm{for~}x>0.$$

Thus, $h(x)$ is strictly increasing and positive for all $x>0$. Hence, $B(x)>0$ for all $x>0$, $A\ge 1$, and $\theta > 0$. Therefore, $f''(x)>0$ on $(0,\infty)$, proving convexity.
\end{proof}

\begin{lemma}\label{lemma:concave}
Let $\theta \ge 1$ and $c>0$. The function $x \mapsto \log^\theta(x+c)$ is concave on $[0,\infty)$ whenever
$c \ge \exp(\theta-1)$.
\end{lemma}

\begin{proof}
Define $f(x):=\log^\theta(x+c)$ for $x\ge 0$. Then
\[
f'(x)=\theta \bigl(\log(x+c)\bigr)^{\theta-1}\frac{1}{x+c}.
\]
Differentiating it again yields
\[
\begin{aligned}
f''(x)
&=\theta\left[
(\theta-1)\bigl(\log(x+c)\bigr)^{\theta-2}\frac{1}{(x+c)^2}
-\bigl(\log(x+c)\bigr)^{\theta-1}\frac{1}{(x+c)^2}
\right]\\
&=\frac{\theta}{(x+c)^2}\bigl(\log(x+c)\bigr)^{\theta-2}\Bigl[(\theta-1)-\log(x+c)\Bigr].
\end{aligned}
\]
Since $\theta\ge 1$ and $x+c>0$, the prefactor $\frac{\theta}{(x+c)^2}\bigl(\log(x+c)\bigr)^{\theta-2}$ is nonnegative whenever $\log(x+c)\ge 0$, and thus $f''(x)\le 0$ holds provided
\[
\log(x+c)\;\ge\;\theta-1\qquad \text{for all }x\ge 0.
\]
The left-hand side is minimized at $x=0$, so it suffices that $\log c \ge \theta-1$, i.e., $c\ge e^{\theta-1}$.
\end{proof}

To construct generalization bounds for heavy-tailed sub-Weibull, the following lemmas are needed, similar to the decorrelation lemma of \cite{chu2023unified}.

\begin{lemma}[A new Young-type inequality of non-convex function]\label{young}
For $x, y \ge 0$, we have
\[
xy \leq 2^{\frac{1}{\theta}} x [\log(x + A)]^{\frac{1}{\theta}} + \exp(y^\theta),
\]
where $A\ge (2^{\lceil \frac{2}{\theta} \rceil-2} \lceil \frac{2}{\theta} \rceil !)^2 \vee 1$ is a positive constant depending on $\theta>0$.
\end{lemma}

\begin{proof}
We discuss two cases below.\\ 

\textbf{Case 1.} If $y \leq 2^{\frac{1}{\theta}} \log(x + A)^{\frac{1}{\theta}}$, the inequality is immediate for $A\ge 1$ since $\exp(y^\theta)>0$.

\textbf{Case 2.} If $y > 2^{\frac{1}{\theta}} \log(x+A)^{\frac{1}{\theta}}$, we have
\[
xy < \left(\exp(\frac{y^\theta}{2}) - A \right) y \leq \left(\exp(\frac{y^\theta}{2}) - \sqrt{A} \right)\left(\exp(\frac{y^\theta}{2}) + \sqrt{A}\right) \leq \exp(y^\theta).
\]
The first inequality is because in this case $x<\exp(\frac{y^\theta}{2})-A$ and we can properly choose $A$ satisfying $y\leq \exp(\frac{y^\theta}{2})+\sqrt{A}$ to make the second inequality hold.

Let \(m=\lceil 2/\theta\rceil\). What does hold for all \(y\ge 0\) and \(\theta>0\) is
\[
\exp\left(\frac{y^\theta}{2}\right)\ \ge\ \frac{y^2}{2^mm!}.
\]
The claim is trivial for $y=0$, so assume $y>0$ and set $x:=y^\theta/2\ge 0$.
By the Taylor expansion of $e^x$ with nonnegative terms, for any integer $m\ge 0$,
\[
e^x=\sum_{k=0}^{\infty}\frac{x^k}{k!}\ \ge\ \frac{x^m}{m!}.
\]
Hence
\[
\exp\!\left(\frac{y^\theta}{2}\right)=e^x
\ \ge\ \frac{1}{m!}\left(\frac{y^\theta}{2}\right)^m
\ =\ \frac{y^{\theta m}}{2^m\,m!}.
\]
We now compare $y^{\theta m}$ and $y^2$.

\textbf{Case 1: $0\le y\le 1$.}
We have $\exp(y^\theta/2)\ge 1$. Also $y^2\le 1$, so
\[
\frac{y^2}{2^m m!}\le \frac{1}{2^m m!}\le 1 \le \exp\!\left(\frac{y^\theta}{2}\right).
\]

\textbf{Case 2: $y\ge 1$.}
Since $m=\left\lceil \frac{2}{\theta}\right\rceil$, we have $\theta m\ge 2$.
Therefore $y^{\theta m}\ge y^2$ for all $y\ge 1$, thus
\[
\exp\!\left(\frac{y^\theta}{2}\right)\ \ge\ \frac{y^{\theta m}}{2^m m!}
\ \ge\ \frac{y^2}{2^m m!}.
\]
So we have
\[
y-\exp\left(\frac{y^\theta}{2}\right) \le y-\frac{y^2}{2^mm!}\le 2^{m-2}m!,~y \ge 0.
\]
$A\ge 2^{m-2}m! \vee 1= (2^{\lceil \frac{2}{\theta} \rceil-2} \lceil \frac{2}{\theta} \rceil !)\vee 1$ is enough.

\end{proof}

\section{Additional Numerical Experiments}\label{app:add_experiments}
This section provides empirical evidence for the main experimental claims. We first verify that the order-preserving power transformation used in Section~\ref{sec:rlhf-experiments} indeed produces a heavy-tailed reward. We then study a distinct heavy-tailed reward construction based on injecting centered Weibull noise and show that the qualitative behavior observed under R\'enyi regularization persists.

\subsection{RLHF experimental details}\label{app:rlhf_details}
We study the effect of R\'enyi order $\alpha$ in~\eqref{optRENYI_penalty} while fixing $\beta=1$ throughout. We sweep $\alpha\in\{1,2,\dots,10\}$, where $\alpha\to 1$ corresponds to the KL regularization. We train in the \texttt{hh-rlhf} training split (50{,}000 prompts) for up to $1500$ optimization steps and evaluate every $10$ steps on 128 prompts that are not used for training. For text generation, we cap the prompt length at 256 tokens and the completion length at 128 tokens. We sample with temperature 1.0 and top-$p$ 0.95, that is, at each decoding step we sample from the smallest set of tokens whose cumulative probability under the model is at least 0.95 (nucleus sampling). During training, we sample 4 completions for each prompt so that GRPO can compare multiple responses to the same prompt and form a group-relative reward signal for the policy update. We use a per-device batch size of 1 and accumulate gradients over 8 steps, which gives an effective batch size of 8, together with a learning rate of $10^{-5}$ and BF16 precision.\footnote{``BF16'' precision refers to the 16-bit \emph{bfloat16} floating-point format, which reduces memory usage and can improve training/inference throughput while typically maintaining numerical stability comparable to 32-bit floats.} During evaluation, we generate with batch size 16 and estimate R\'enyi divergence on a 16-prompt subset to reduce computation. For reward scoring, both reward models use a maximum input length of 512 tokens, with longer inputs truncated. We also recalibrate the proxy reward using mean- and scale statistics estimated from 256 prompts sampled from the reference policy so that reward values are on a comparable scale across runs.\\

\subsection{Tail behavior under the order-preserving power transformation}\label{sec:tail-behavior-power-transform}

As discussed in Section~\ref{sec:rlhf-experiments}, the RLHF experiments are conducted in a heavy-tailed reward regime induced by
an order-preserving power transform. In Figure~\ref{gold_proxy_trans}, we
complement that visual evidence with a quantitative diagnostic via the empirical MGF.

For the original or transformed reward samples $\{r_i\}$, we
compute
$
M(t)=\frac{1}{n}\sum_{i=1}^n e^{tr_i},~t\in\mathbb{R}.
$
Its growth for moderate values of $|t|$ provides a useful finite-sample summary
of tail amplification. $M(t)$ changes much more rapidly and explore as $t$ moves away from zero.

Figure~\ref{gold_proxy_MGF} shows that the transformed rewards exhibit
orders-of-magnitude larger empirical MGFs than the original rewards. For
$\beta=1$ and $\alpha=1$, at $t=-0.4$, the empirical MGF of the transformed
reward is approximately $5.77\times 10^{6}$, about $3.2\times 10^{6}$ times
larger than that of the original reward ($1.80$). For $\alpha=2$, the contrast
is even stronger: at $t=-0.4$, the ratio between the transformed and original
MGFs is approximately $8.34\times 10^{13}$. The logarithmic growth rate
$d\log M(t)/dt$ shows the same pattern. At $t=-0.4$, this quantity is
$-47.5$ for $\alpha=1$ and $-91.3$ for $\alpha=2$ after
transformation, compared with only $-1.47$ and $-1.57$ for the corresponding
original rewards. Taken together, these diagnostics reinforce the main-text
claim in Section~\ref{sec:rlhf-experiments}: the order-preserving power
transformation substantially amplifies tail behavior and thereby produces the
heavy-tailed reward regime used to study the robustness of
R\'enyi-regularized RLHF.

\begin{figure}[H]
\centering
\includegraphics[scale=0.225]{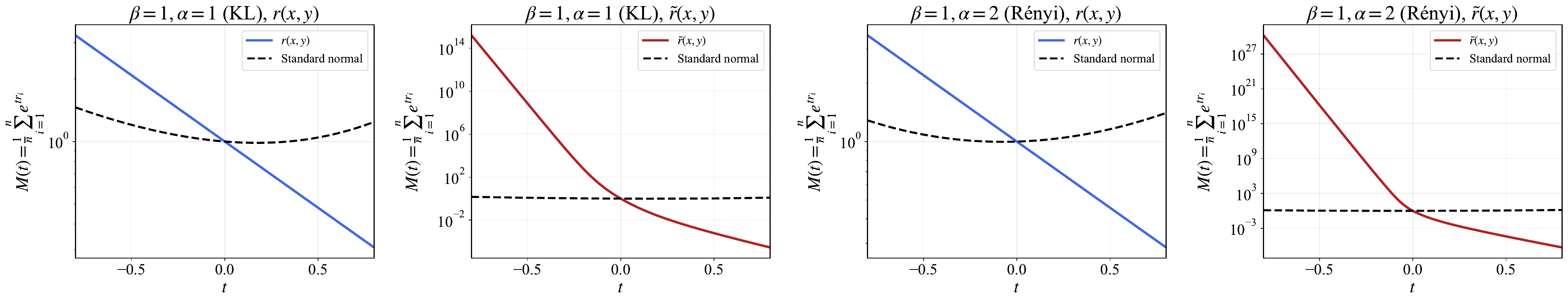}
\caption{\small Empirical MGFs comparison for the original and transformed reward. Left: $\beta=1$ with $\alpha=1$ (KL). Right: $\beta=1$ with
$\alpha=2$ (R\'enyi). Blue curves correspond to the original reward $r(x,y)$,
and red curves correspond to the transformed reward
$\tilde r(x,y)=\operatorname{sign}(r(x,y))\,|r(x,y)|^{8}$. The much faster
growth of the empirical MGF for the transformed rewards is consistent with stronger tail amplification.}
\label{gold_proxy_MGF}
\end{figure}

\subsection{Alternative heavy-tailed reward construction by injecting Weibull noise}\label{app:weibull-noise}

In addition to the power transform used in Section~\ref{sec:rlhf-experiments}, we consider a second heavy-tailed reward construction obtained by adding centered Weibull noise to the reward. Let
\[
r_0(x,y)=\langle \theta^{*},\psi(x,y)\rangle
\]
denote the clean reward, where $\psi(x,y)$ is the penultimate-layer representation of the same proxy reward model used in Section~\ref{sec:rlhf-experiments}, namely Qwen2.5-0.5B, and $\theta^{*}$ is the parameter vector of its trained final linear layer. We define the observed reward by
\[
r(x,y)=r_0(x,y)+\varepsilon(x,y),
\qquad
\varepsilon(x,y)=Z-\mathbb{E}[Z],
\qquad
Z\sim \mathrm{Weibull}(k,\eta(x,y)).
\]
Since $\mathbb{E}[Z]=\eta(x,y)\Gamma(1+1/k)$, the perturbation is centered and therefore preserves the conditional mean reward. We set
\(
k\in\{0.2,0.3,0.4,0.5\},
\)
with smaller $k$ corresponding to heavier-tailed noise, and take $\eta(x,y)$ to be the estimated standard error of the reward. We use the same penalized objective~\eqref{optRENYI_penalty}, fix $\beta=1$, and vary the R\'enyi order over
\(
\alpha\in\{2,4,6,8,10\}.
\)
All other implementation details are the same as those in Section~\ref{sec:rlhf-experiments}.

\begin{figure}[h]
\centering 
\includegraphics[scale=0.24]{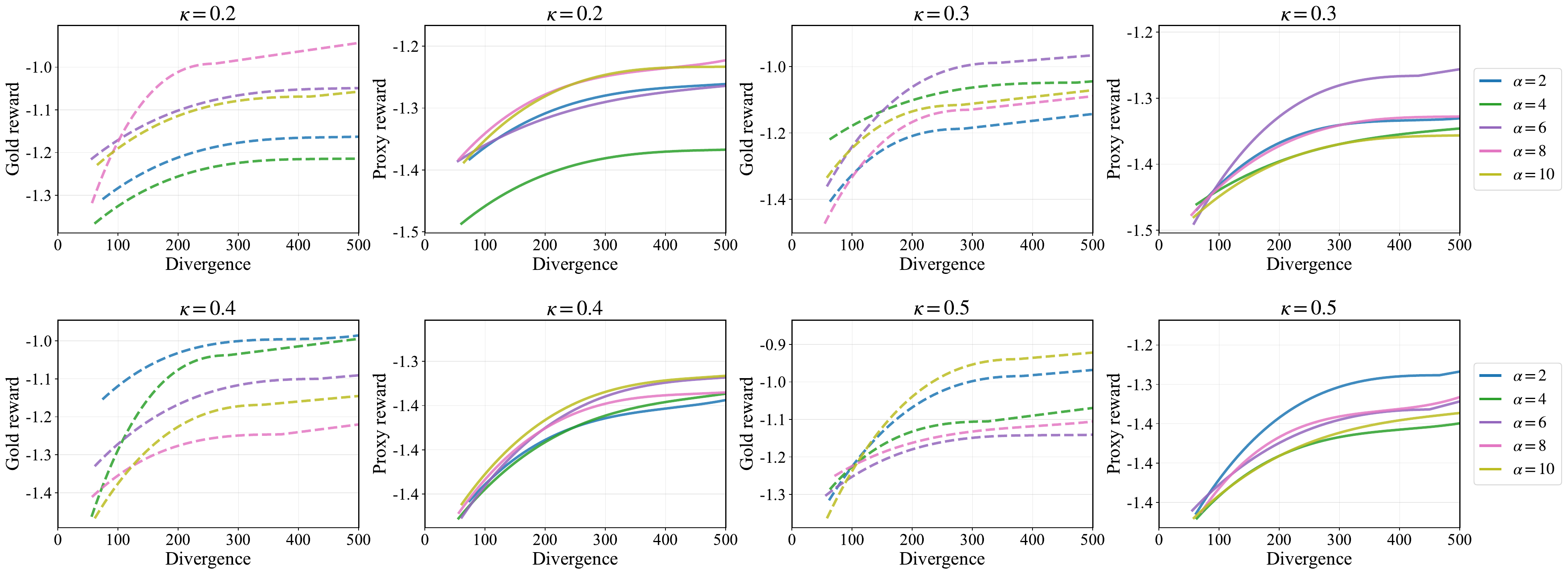}
\caption{\small
Reward--divergence relationship under the centered Weibull-noise reward with shape parameters 0.2, 0.3, 0.4, and 0.5, ordered from left to right and top to bottom.
Across all Weibull shape parameters, the reward--divergence curves remain stable: the gold reward typically increases and then levels off, rather than exhibiting a sharp collapse, indicating that higher-order R\'enyi regularization provides a reliable trust-region with $\beta=1$.
}
\label{divergence_vs_reward_Weibull}
\end{figure}

Figures~\ref{divergence_vs_reward_Weibull} and~\ref{proxy_vs_gold_Weibull} show that the qualitative behavior observed in the main text persists under this perturbation model. Across all tested noise levels, the reward--divergence curves remain stable and the proxy--gold relationship remains monotone over the observed range. Although smaller $k$ produces heavier-tailed reward noise, it has the stabilizing effect under higher-order R\'enyi regularization. These results reinforce the main conclusion of Section~\ref{sec:rlhf-experiments}: the favorable empirical behavior of R\'enyi regularization is not an artifact of the power transformation, but persists under a distinct heavy-tailed reward construction.

\begin{figure}[H]
\centering 

\includegraphics[scale=0.45]{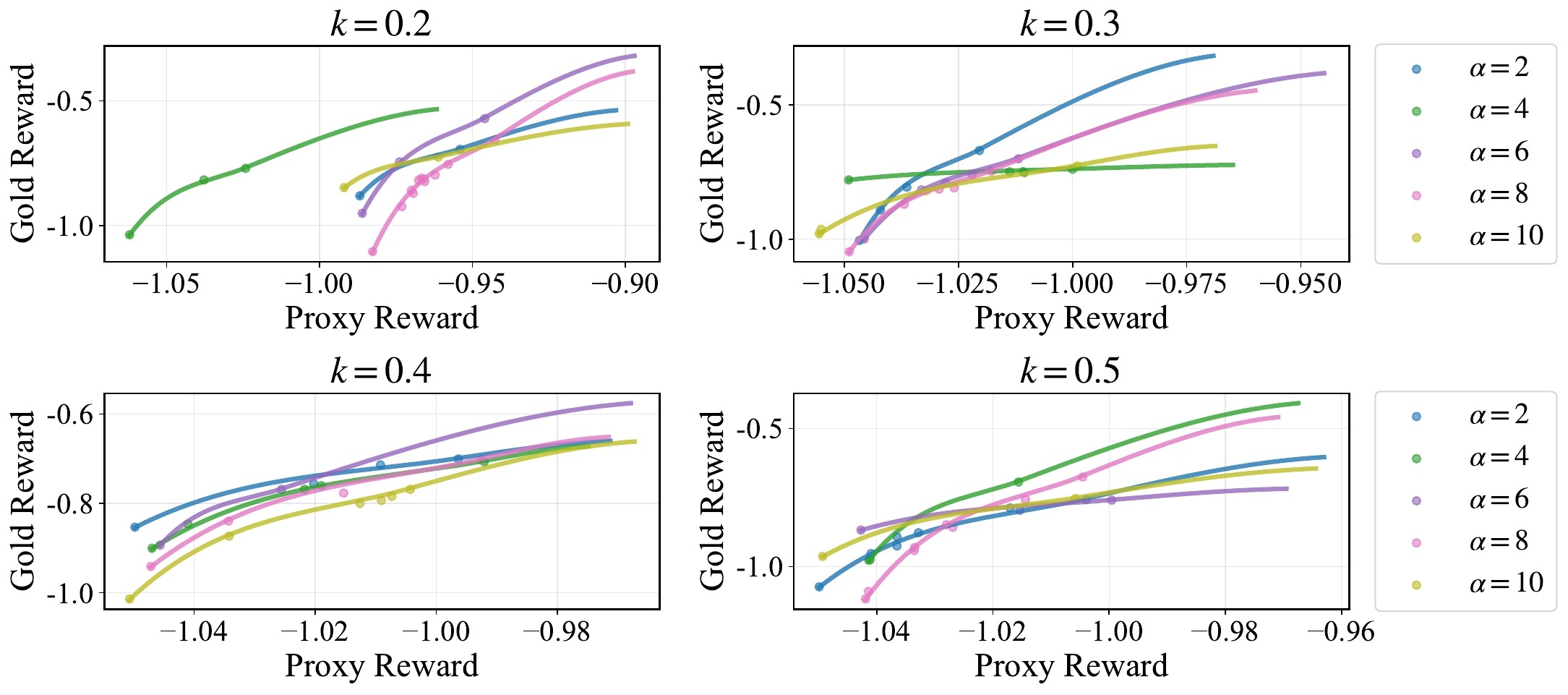}

\caption{\small
Proxy--proxy-gold reward relationship under the Weibull-noise reward construction with $\beta=1$ and shape parameters 0.2 to 0.5, ordered from left to right and top to bottom.
Across all noise levels, the proxy--gold curves remain monotone increasing over the observed range, showing that improvements in the proxy reward continue to translate consistently into improvements in the gold reward even under heavier-tailed Weibull.
}
\label{proxy_vs_gold_Weibull}
\end{figure}

\end{document}